\documentclass[11pt,twoside]{article}

\usepackage{fullpage}

\usepackage{amsmath,amssymb,fullpage,graphicx,xcolor,mathtools,nicefrac,comment}
\usepackage[shortlabels]{enumitem}

\usepackage{amssymb,amsmath,comment}
\usepackage[most]{tcolorbox}
\usepackage{subfig}

\usepackage{amsthm}
\makeatletter
\def\th@plain{%
  \thm@notefont{}% same as heading font
  \itshape % body font
}
\def\th@definition{%
  \thm@notefont{}% same as heading font
  \normalfont % body font
}
\makeatother

\usepackage[shortlabels]{enumitem}

\usepackage{thmtools,thm-restate}

\usepackage{hyperref}

\usepackage{amsmath,amssymb}
\usepackage{verbatim,float,url}
\usepackage{graphicx,psfrag}
\usepackage[numbers]{natbib}
\usepackage{xcolor}
\usepackage{microtype}
\usepackage{xparse}
\usepackage{xspace}
\usepackage{tikz,textcomp,thmtools,nameref,cleveref}
\usetikzlibrary{positioning}
\usepackage{tikz-qtree}

\usepackage{algpseudocode,algorithmicx,algorithm}

\usepackage[font=small]{caption}

\newtheorem{definition}{Definition}
\newtheorem{theorem}{Theorem}

\newtheorem{counterexample}{Counterexample}

\newtheorem{dhruvprop}{Proposition}
\newtheorem{dhruvdef}[theorem]{Definition}
\newtheorem{dhruvthm}{Theorem}
\newtheorem{dhruvlemma}{Lemma}

\DeclareMathOperator*{\argmax}{arg\,max}
\DeclareMathOperator*{\argmin}{arg\,min}

\newcommand{\qeddhruv}{$\blacksquare$}

\newcommand{\compclass}{\mathcal{C}_T}
\newcommand{\bigo}{\widetilde{\mathcal{O}}}
\newcommand{\mem}{m}
\newcommand{\numact}{K}
\newcommand{\ActSet}{\mathcal{X}}
\newcommand{\pReg}{\mathcal{R}^{\text{pol}}_{\compclass}}
\newcommand{\cpReg}{\mathcal{R}^{\text{cp}}_T}
\newcommand{\tReg}{\mathcal{R}^{\text{trad}}_{\compclass}}
\newcommand{\E}{\mathbb{E}}
\newcommand{\bigtheta}{\widetilde{\Theta}}
\newcommand{\bigomega}{\widetilde{\Omega}}
\newcommand{\compconstant}{\ActSet_{\text{const}}}
\newcommand{\pRegconst}{\mathcal{R}^{\text{pol}}_{\compconstant}}
\newcommand{\tRegconst}{\mathcal{R}^{\text{trad}}_{\compconstant}}

\begin{document}

\begin{center}
{\bf{\LARGE{Complete Policy Regret Bounds for Tallying Bandits}}}
%\cmcomment{this looks fine, another possibility:}

%{\bf{\LARGE{Optimality and Sub-optimality in Estimating Bivariate Isotonic Matrices with Unknown Permutations}}}

\vspace*{.2in}

{\large{
\begin{tabular}{cccc}
Dhruv Malik & Yuanzhi Li & Aarti Singh
%p1$^\star$ & p2$^\dagger$ & p3$^{\dagger, \ddagger}$
\end{tabular}
}}
\vspace*{.2in}

%\begin{tabular}{cc}
%Machine Learning Department$^{\dagger}$ & Department of Electrical Engineering and Computer Sciences$^{\ddagger}$ \\
%Carnegie Mellon University & University of California, Berkeley \\
%Pittsburgh, PA 15213 & Berkeley, CA 94720
%\end{tabular}
\begin{tabular}{c}
Machine Learning Department \\
Carnegie Mellon University
\end{tabular}

\vspace*{.2in}

\today

\end{center}
\vspace*{.2in}
%%%%%%%%%%%%%%%%%%%%%%%%%%%%%%%%

\begin{abstract}
Policy regret is a well established notion of measuring the performance of an online learning algorithm against an adaptive adversary. We study restrictions on the adversary that enable efficient minimization of the \emph{complete policy regret}, which is the strongest possible version of policy regret. We identify a gap in the current theoretical understanding of what sorts of restrictions permit tractability in this challenging setting. To resolve this gap, we consider a generalization of the stochastic multi armed bandit, which we call the \emph{tallying bandit}. This is an online learning setting with an $m$-memory bounded adversary, where the average loss for playing an action is an unknown function of the number (or tally) of times that the action was played in the last $m$ timesteps. For tallying bandit problems with $\numact$ actions and time horizon $T$, we provide an algorithm that w.h.p achieves a complete policy regret guarantee of $\bigo ( m \numact \sqrt{T} )$, where the $\bigo$ notation hides only logarithmic factors. We additionally prove an $\bigomega(\sqrt{ m \numact T})$ lower bound on the expected complete policy regret of any tallying bandit algorithm, demonstrating the near optimality of our method.
    \end{abstract}

\section{Introduction}
When decision making algorithms are deployed in the real world, the reward associated with choosing a decision is rarely static. Instead, an algorithm's decision impacts the state of its environment, which in turn influences the quality of that same decision in the future. For instance, in recommender systems such as YouTube and Netflix, the choice to recommend a type of content is often instrumental in shaping the preferences of the user for that content genre. This creates a feedback loop between an algorithm and its environment, and results in a complex back and forth interaction.

Such dynamic and interactive settings are well modeled as online learning problems, where a player competes against an adaptive adversary. To measure the performance of the player, most of the literature on online learning has focused on a performance metric called the \emph{traditional regret}~\citep{auer02, flaxman05, abernethy08, hazan12}. However, a significant line of work has established that when the adversary is adaptive, the traditional regret is a poor indicator of the performance of an algorithm~\citep{mehrav02, arora12, cesa-bianchi13, heidari16, lindner21}. Instead, one typically opts for a stronger performance metric, known as \emph{policy regret}. The policy regret accumulated over a time horizon $T$ is defined with respect to a competitor class $\compclass$ of deterministic policies (or length $T$ action sequences). The policy regret with respect to $\compclass$, which we denote $\pReg$, compares the algorithm's cumulative loss to that of the best policy in $\compclass$.

%The traditional regret and policy regret accumulated over a time horizon $T$ are each defined with respect to a competitor class $\compclass$ of deterministic policies (or length $T$ action sequences). Critically, the policy regret compares the algorithm's cumulative loss to that of the best policy in $\compclass$. By contrast, the traditional regret only compares the player's loss to the minimum loss that could have been collected by playing a single policy in $\compclass$ on the same sequence of rewards as generated by the player, and is a much weaker notion of performance. \\

Much of the prior work on policy regret has focused on the restrictive assumption that $\compclass$ contains only those action sequences that repeatedly play the same action~\citep{arora12, cesa-bianchi13, dekel14, arora18}. We allude to the policy regret with this choice of $\compclass$ as the \emph{constant action policy regret}. The strongest possible version of policy regret is when $\compclass$ is the complete policy class (i.e., the set of all deterministic policies), and we abbreviate this as the \emph{complete policy regret}. This challenging setting has recently received attention from the online learning community~\citep{heidari16, seznec20, lindner21}. On the other hand, this performance metric is equivalent to the one that is standard in the closely related field of reinforcement learning, where a vast literature explores how to efficiently maximize cumulative reward~\citep{kearns99generative, sutton00, kakade03, jin18, wang20, malik21}.

Unfortunately, prior work has shown that without restrictions on the adversary, obtaining non-trivial guarantees on even the constant action policy regret is impossible~\citep{arora12}. Hence, to attain meaningful guarantees on the complete policy regret (CPR), it is necessary to restrict the adversary. Prior literature on policy regret studies different competitor classes $\compclass$, along with varying types of restrictions on the adversary, to demonstrate non-trivial guarantees on the corresponding policy regret $\pReg$. We comprehensively survey these restrictions, and identify a gap in the current theoretical understanding of when it is possible to attain meaningful guarantees on CPR. To resolve this gap, we make the following contributions:
\begin{itemize}
\item We introduce an online learning setting known as the \emph{tallying bandit}. Here the average loss for playing an action is a function of the number (or tally) of times that action was played in the last $m$ timesteps. The stochastic multi armed bandit (sMAB) is a special case of the tallying bandit, via a choice of $m=1$. From a more practical angle, we view the tallying bandit as a step towards handling feedback loops that arise in applications such as recommender systems.
\item For tallying bandit problems with $\numact$ actions and time horizon $T$, we provide an algorithm, that given any $\delta \in (0,1)$, achieves with probability at least $1 - \delta$ a complete policy regret guarantee of $\bigo \left( m \numact \sqrt{T} \log \left( T, m, \numact, 1/\delta \right) \right)$.
\item We complement our algorithmic development with an $\bigomega(\sqrt{\mem \numact T})$ minimax lower bound on the expected complete policy regret of any method designed for tallying bandits. This demonstrates the near optimality of our algorithm.
\end{itemize}

\section{Problem Formulation}
\subsection{Online Learning \& Complete Policy Regret}
We begin by providing a generic formulation of online learning against adaptive adversaries, following rather closely the description of Arora et al.~\citep{arora12}. An online learning problem with time horizon $T$ and action set $\ActSet$ is an iterative game between a player and an adaptive adversary. Throughout, we let $\numact$ denote the cardinality of $\ActSet$. Before the game begins, the adversary fixes a sequence of history dependent loss functions $\{ f_t \}_{t=1}^T$, where $f_t$ maps $\ActSet^t$ to the interval $[0,1]$. At each timestep $t$ of the game, the player chooses an action $a_t \in \ActSet$. In the bandit feedback model, the player then observes the loss value $f_t(a_{1:t})$, where we have used $a_{1:t}$ as shorthand for $a_1, a_2 \dots a_t$. By contrast, in the full information model, the player observes $f_t(a_{1:t-1}, x)$ for all $x \in \ActSet$.

The cumulative loss experienced by the player during this game is $\sum_{t=1}^T f_t(a_{1:t})$. Note that this is a random variable, since the player's strategy can be random. In order to evaluate the performance of the player, we compare this cumulative loss to a baseline. In particular, we let $\compclass \subseteq \ActSet^T$ be a competitor class of policies (or length $T$ action sequences). Given some $\compclass$, one typically measures the player's performance via either the \emph{policy regret}, which we denote $\pReg$, or the \emph{traditional regret}, which we denote $\tReg$. The policy regret~\citep{arora12, cesa-bianchi13, arora18} is defined as
\begin{equation}
\label{eqn:policy_regret_main}
\pReg = \sum_{t=1}^T f_t(a_{1:t}) - \min_{(y_1, y_2 \dots y_T) \in \compclass} \sum_{t=1}^T f_t(y_1, y_2 \dots y_t).
\end{equation}
Notably, this definition differs substantially from the traditional regret~\citep{auer02, flaxman05, abernethy08, hazan12}, given by
\begin{equation}
\label{eqn:traditional_regret_main}
\tReg =  \sum_{t=1}^T f_t(a_{1:t}) - \min_{(y_1, y_2 \dots y_T) \in \compclass} \sum_{t=1}^T f_t(a_{1:t-1}, y_t).
\end{equation}
\noindent In the aforementioned adversarial online learning setup, the traditional regret lacks meaningful interpretation. Instead, one opts for the policy regret to measure the player's performance. We refer the interested reader to Arora et al.~\citep{arora12} for further details on the motivation for this choice.

Let $\compconstant$ denote the set of constant action sequences, so that $\compconstant = \{ (x, x \dots x) \text{ s.t. } x \in \ActSet \}$. We refer to $\pRegconst$ and $\tRegconst$ respectively as the \emph{constant action policy regret} and \emph{constant action traditional regret}. The constant action policy regret (and hence the constant action traditional regret) yields a weak measure of performance, since we are comparing the player's performance to a very restricted baseline policy class. Expanding the competitor class $\compclass$ yields stronger notions of performance. In our paper, we are interested in the challenging setting where $\compclass$ is the complete policy class (the set of all length $T$ action sequences), or equivalently where $\compclass = \ActSet^T$. This choice of $\compclass$ in Eq.~\eqref{eqn:policy_regret_main} yields the strongest version of policy regret, and we refer to it as \emph{complete policy regret (CPR)}, and denote it by $\cpReg$. An \emph{optimal policy} is a policy in $\ActSet^T$ that minimizes the cumulative loss. We will often use the terminology ``efficiently minimize the CPR'', which means to obtain a CPR bound that is sublinear in $T$ and at most polynomial in all other problem dependent parameters. Our exclusive focus is on statistical (rather than computational) efficiency.

\subsection{Restricting The Adversary}
\noindent Prior work due to Arora et al.~\citep{arora12} has shown that without any restrictions on the adversary, and even when $\compclass = \compconstant$, for any player there exists an adversary such that the player's constant action policy regret satisfies $\pRegconst = \bigomega(T)$. To prove this lower bound, Arora et al.~\citep{arora12} construct an adversary that is wholly unrestricted and hence extremely powerful. Thus, to obtain non-trivial upper bounds on even the constant action policy regret, it is necessary to weaken the adversary. One natural type of restriction that has been well studied in prior work, is to restrict the memory of the adversary~\citep{arora12, cesa-bianchi13, arora18}.
\begin{dhruvdef}
We say that an adversary is $m$-memory bounded if for all $t \geq m$, all $a_{1:t} \in \ActSet^t$, all $a'_{1:t-m} \in \ActSet^{t-m}$ and all $f_t$ we have that
$$
f_t(a_{1:t}) = f_t(a'_1, a'_2 \dots a'_{t-m}, a_{t-m+1} \dots a_t ).
$$
\end{dhruvdef}
Hence, an $m$-memory bounded adversary is only permitted to define its loss function based on the player's most recent $m$ actions. Prior work has shown that when $m$ is sublinear in $T$ and $\compclass$ is sufficiently restricted (informally, when $\compclass$ equals or is only slightly larger than $\compconstant$), then the player can achieve policy regret $\pReg$ that is sublinear in $T$~\citep{arora12, cesa-bianchi13, dekel14, arora18}.

When $m=1$, then the adversary is oblivious, and we overload notation and write $f_t(a_{1:t}) = f_t(a_t)$. Notably, for a fixed $\compclass$, the policy regret $\pReg$ equals the traditional regret $\tReg$ in this scenario. It is well known that against an oblivious adversary, a player can achieve sublinear constant action policy regret~\citep{arora12}. Hence, it is natural to question whether a player can achieve sublinear CPR, when the adversary is oblivious. We show via the following counterexample that this is impossible. Note that the result of the counterexample holds even in the full information feedback model (as opposed to just bandit feedback). A similar result is given by Mohri and Yang~\citep{mohri18}.

\begin{counterexample}
\label{counterexample1}
Let $\ActSet = \{ x_1, x_2 \}$. To define its sequence of loss functions, the adversary first samples a bit string $b$ uniformly at random from $\{ 0, 1 \}^T$. For each $t$, it then defines
$$
f_t(x_1) = b_t \text{ and } f_t(x_2) = 1 - b_t.
$$
Attaining sublinear complete policy regret is then equivalent to making a sublinear number of mistakes when guessing the value of $b_t$. Since this is impossible, we have that $\E \left[ \cpReg \right] = \bigomega(T)$, where the expectation is over the sampling of $b$ and the player's (possibly randomized) strategy.
\end{counterexample}

\noindent Crucially, the above counterexample relies on the fact that even though the adversary is oblivious, its loss functions $f_t, f_{t'}$ for $t \neq t'$ are time varying and are constructed independent of each other. In this scenario, the player cannot predict anything about $f_t$ via knowledge of $f_{t'}$ for $t' < t$. To evade such counterexamples, a different type of restriction on the adversary's power is to ensure that some knowledge of $f_{t'}$ leaks some information about $f_t$. This motivates the following definition. 

\begin{dhruvdef}
An $m$-memory bounded adversary is said to be $g$-restricted if the following is true. For each action $x \in \ActSet$, there exists a base function $g_x: \cup_{m'=1}^m \ActSet^{m'} \to [0,1]$, such that
$$
f_t(a_{1:t}) \equiv f_t(a_{\max \{ 1, t-m+1 \}:t}) = g_{a_t}(a_{\max \{ 1, t-m+1 \}:t}).
$$
\end{dhruvdef}
Although we are not aware of prior work on online learning that uses $g$-restricted adversaries without additional restriction, in the sequel we will discuss prior work that consider $g$-restricted adversaries with significant additional restrictions~\citep{heidari16, levine17, seznec19, seznec20, lindner21, awasthi22}. For now, we note that a $g$-restricted adversary must define loss functions whose value for a fixed input cannot vary with time. With such a restriction, the player can learn information about each $g_x$ (and hence about $f_t$) as the game progresses, and this enables the player to choose better actions over time. This restriction thus precludes the setting of Counterexample~\ref{counterexample1}.

When the adversary is $g$-restricted, it is straightforward to achieve CPR bounds that are sublinear in $T$ and depend exponentially on $m$. However, throughout our paper we interested in efficiently minimizing CPR, which means we desire bounds that scale polynomially with $m$. Unfortunately, there exist online learning games where the adversary is $m$-memory bounded and $g$-restricted, but it is impossible to efficiently minimize CPR. We demonstrate this in the following counterexample, which holds even in the full information feedback model (as opposed to just bandit feedback).

\begin{counterexample}
\label{counterexample2}
Let $\ActSet = \{ x_1, x_2 \}$. Sample a tuple $b$ of length $m-1$ uniformly at random from $\ActSet^{m-1}$. Define $g_{x_1}: \cup_{m'=1}^m \ActSet^{m'} \to [0,1]$ as 
$$
g_{x_1}(a_{1:m'}) = 1 \text{ if } m' < m \text{ and } g_{x_1}(a_{1:m}) = 1 - \prod_{i=1}^{m-1} \mathbb{I} \left( a_i = b_i \right).
$$
Also define $g_{x_2} = 1$. Via the base functions $g_{x_1}, g_{x_2}$ we define the adversary's loss functions as
$$
f_t(a_{1:t}) \equiv f_t(a_{t-m+1:t}) = g_{a_t}( a_{t-m+1:t}).
$$
The policy that cyclically plays actions $b_1, b_2 \dots b_{m-1}, x_1$ suffers a loss of zero at least once every $m$ timesteps. Meanwhile, suffering zero loss for the player is at least as hard as identifying $b$, and a standard ``needle in the haystack'' argument~\citep{du20lowerbound} shows that this requires $\bigomega(2^m)$ timesteps. Hence we have that $\E \left[ \cpReg \right] = \bigomega(\min \{ 2^m, T \} / m)$, where the expectation is over the sampling of $b$ and the player's (possibly randomized) strategy.
\end{counterexample}

\noindent This counterexample demonstrates that even if the adversary is $m$-memory bounded and $g$-restricted, any player suffers CPR that scales exponentially with $m$. So, further restrictions on the adversary are necessary. A natural restriction is to enforce that each $g_x$ has special structure. This is precisely the approach taken by works on \emph{rotting} bandits~\citep{heidari16, levine17, seznec19, seznec20}, \emph{improving} bandits~\citep{heidari16}, \emph{single peaked} bandits~\citep{lindner21} and \emph{congested} bandits~\citep{awasthi22}. Concretely, these works use base functions $\{ g_x \}_{x \in \ActSet}$ that have the following special ``tallying'' structure.

\begin{dhruvdef}
An $m$-memory bounded and $g$-restricted adversary is said to be $h$-tallying, if for each $x \in \ActSet$ there exists $h_x: \{ 1, 2 \dots m \} \to [0,1]$ such that
$$
f_t(a_{1:t}) \equiv f_t(a_{\max \{ 1, t-m+1 \}:t}) = g_{a_t}(a_{\max \{ 1, t-m+1 \}:t}) = h_{a_t} \left( \sum_{t'=\max \{ 1, t-m+1 \}}^t \mathbb{I}(a_{t'} = a_t) \right).
$$
\end{dhruvdef}

\noindent As discussed by the aforementioned works, this tallying structure is often a natural model in practice. For instance, Heidari et al.~\citep{heidari16} discuss a crowdsourcing setting where an agency utilizes workers to repeatedly perform the same task (such as classifying images) at each timestep. The agency picks a worker at each timestep, with the goal of picking a sequence of workers that makes the fewest number of mistakes when performing the task. Here, it is reasonable that an individual worker's performance changes as an (unknown) function of the \emph{number} of times that the worker has already performed the task (for example, due to fatigue), thus motivating the tallying structure.

We emphasize that in addition to assuming the adversary is $h$-tallying, the aforementioned works of Heidari et al.~\citep{heidari16}, Levine et al.~\citep{levine17}, Seznec et al.~\citep{seznec19}, Seznec et al.~\citep{seznec20}, Lindner et al.~\citep{lindner21} and Awasthi et al.~\citep{awasthi22} make \emph{supplemental} benign assumptions on the structure of the functions $\{ h_x \}_{x \in \ActSet}$, as we detail in Section~\ref{sec:related_work}. For instance, the rotting bandit setting of Heidari et al.~\citep{heidari16} assumes that $h_x$ is an increasing function for each $x \in \ActSet$. Under this benign assumption, they provide algorithms that efficiently minimize CPR. Notably, such strong assumptions on $h_x$ enable this line of work to (often) tackle the case where $m=T$ (although algorithms designed for this case generally do not handle $m < T$), and additionally enables these works to (often) handle the more difficult scenario of when the losses are \emph{stochastically} observed.

This exposes a gap in our understanding of when one can efficiently minimize CPR. In particular, it remains unclear whether we can attain this goal for $h$-tallying adversaries where we make \emph{no} assumptions on the structure of each $h_x$. This motivates the following question.

\begin{center}
\emph{Assume the adversary is $m$-memory bounded, $g$-restricted and $h$-tallying. Without any assumptions on the functions $\{ h_x \}_{x \in \ActSet}$, and in the bandit feedback model with (possibly) stochastically observed losses, when is it possible to efficiently minimize the complete policy regret?}
\end{center}

\noindent The remainder of this paper is devoted to resolving this question. To this end, in the sequel we define the \emph{tallying bandit}, and provide upper and lower bounds on the achievable CPR in this setting.

\section{Tallying Bandits}
\label{sec:tallying_bandits}
Let us formally introduce the tallying bandit setting.
\begin{definition}
\label{def:tallying_bandit}
An online learning game is an $(m,g,h)$-tallying bandit if the adversary is $m$-memory bounded, $g$-restricted and $h$-tallying, and if after playing action $a_t$ the player observes a random variable $\widetilde{h}_{a_t}(y_t) \in [0,1]$ satisfying
$$
\E \left[ \widetilde{h}_{a_t}(y_t) \right] = h_{a_t} \left( y_t \right) = g_{a_t}(a_{\max \{ 1, t-m+1 \}:t}) = f_t(a_{1:t}),
$$
where $y_t = \sum_{t'=\max \{ 1, t-m+1 \}}^t \mathbb{I}(a_{t'} = a_t)$.
\end{definition}

\noindent We assume the cardinality $\numact$ of the action set $\ActSet$ is finite, and also that $m$ is known (although we discuss how to relax this in Section~\ref{sec:discussion}). With this definition of the setting in hand, we can restate our goal of efficiently minimizing the CPR. Concretely, we desire an algorithm, which when given an $(m,g,h)$-tallying bandit problem, has a CPR bound that is polynomial in $\numact, m$ and is sublinear in $T$. The tallying bandit strictly generalizes the well studied stochastic multi armed bandit (sMAB)~\citep{lai85, auer02stochastic}, simply via a choice of $m=1$. Hence, we generalize the study of sMAB to $m>1$. We remark that tallying bandit is a special case of the \emph{rested} bandit, a general framework for nonstationary MAB where the reward of an arm evolves when it is pulled, and we defer detailed discussion of this to Section~\ref{sec:related_work}.

Recall that in sMAB, the optimal policy plays the same optimal arm at each timestep. Hence, in sMAB obtaining zero constant action traditional regret is equivalent to obtaining zero constant action policy regret and also zero CPR. Given that the tallying bandit is highly structured and generalizes sMAB via $m=1$, it is natural to question whether minimizing constant action policy (or traditional) regret implies minimizing CPR. The following counterexample answers this question in the negative when $m > 1$. More generally, this counterexample shows that even when the adversary is restricted (as in tallying bandits), minimizing the constant action policy (or traditional) regret can lead to solutions whose cumulative loss is $\bigomega(T)$ larger than the minimum achievable total loss.

\begin{counterexample}
\label{counterexample:not_constant_action}
Let $\ActSet = \{ x_1, x_2 \}$ and $m=2$. Define $h_{x_1}(1) = h_{x_2}(1) = 0$ and $h_{x_1}(2) = h_{x_2}(2) = 1$. A policy that fixes either of the two actions, and then plays this action at every timestep, incurs cumulative loss $T-1$ but has zero constant action policy regret and zero constant action traditional regret. Meanwhile, the complete policy regret of this policy is $T-1$, since the optimal policy that alternates playing actions $x_1$ and $x_2$ incurs zero cumulative loss.
\end{counterexample}

\noindent Thus far, our motivation for the tallying bandit setting has been primarily theoretical, to resolve the gap in our understanding of when we can efficiently minimize CPR. Nevertheless, in similar vein to Heidari et al.~\citep{heidari16}, Lindner et al.~\citep{lindner21} and Awasthi et al.~\citep{awasthi22}, we believe that the tallying bandit is a simple approximation for various practical settings. For instance, in recommender systems the reward associated with an action is rarely static, because the stimulus of recommended content influences user preferences~\citep{cosley03, sinha16}. Moreover, literature on psychology and cognition suggests that humans often forget prior stimuli and do not always encode them in permanent memory~\citep{klatzky80, chessa07}. Thus, one way to model a user's preferences is via an (unknown) function of the \emph{number} of times a content genre has been recommended in a \emph{recent} time interval, motivating both the $h$-tallying structure as well as bounded $m$. Nevertheless, the tallying bandit is just one plausible model, and we suggest possible extensions in Section~\ref{sec:discussion}.

Let us now discuss potential avenues for efficiently minimizing CPR in tallying bandit problems. One approach is to observe that any tallying bandit problem can be cast as a reinforcement learning (RL) problem, where each state corresponds to a sequence of actions taken in the last $m$ timesteps. However, methods for solving such RL problems typically scale with the cardinality of the state space~\citep{azar13, azar17, jin18}, and such approaches would suffer $\bigomega(\numact^m)$ CPR.

Hence it is necessary to leverage the additional properties of tallying bandit problems. To gain intuition, let us consider the simplified setting of deterministic bandit feedback, where for each $t, a_{1:t} \in \ActSet^t$ the player observes $f_t(a_{1:t})$ with no noise. Since the loss functions $f_t$ are fully defined by the functions $\{ h_x \}_{x \in \ActSet}$, it is natural to consider the following algorithm, which we denote $\mathbb{ALG}_{\text{det}}$. First, the algorithm queries $h_x(y)$ at each $(x, y) \in \ActSet \times \{ 1, 2 \dots m \}$. This yields full information about the loss functions $f_t$. Then, the algorithm plans offline an optimal sequence of actions for the remaining timesteps. $\mathbb{ALG}_{\text{det}}$ is formally specified as Algorithm~\ref{alg:alg_det} in Appendix~\ref{app:proposition1}, and the following result shows that its CPR is minimax optimal (upto constant factors).

\begin{dhruvprop}
\label{prop1}
Consider any $(m,g,h)$-tallying bandit problem with deterministic bandit feedback. Then the complete policy regret of Algorithm~\ref{alg:alg_det} ($\mathbb{ALG}_{\text{det}}$) is almost surely upper bounded as $\cpReg \leq (m+1) \numact$. Moreover, there exists an $(m,g,h)$-tallying bandit problem, such that the (expected) complete policy regret of any (possibly randomized) algorithm on this problem with deterministic feedback is lower bounded as $\E \left[ \cpReg \right] \geq m \numact / 128$.
\end{dhruvprop}

\noindent The proof of Proposition~\ref{prop1} is included in Appendix~\ref{app:proposition1}. Due to the optimality of $\mathbb{ALG}_{\text{det}}$ given deterministic feedback, it is reasonable to extend it to handle stochastic feedback. One natural way to do so is via an ``explore then exploit'' modification, which has been studied even for sMAB~\citep{slivkins19}. Consider the following ``explore then exploit'' algorithm, which we denote $\mathbb{ALG}_{\text{stoch}}$. It queries repeatedly to receive stochastic realizations of $h_x(y)$ at each $(x, y) \in \ActSet \times \{ 1, 2 \dots m \}$, and constructs tight confidence intervals for these $h_x(y)$. It then uses these estimated values of $h_x(y)$ to plan offline an optimal sequence of actions for the remaining timesteps. $\mathbb{ALG}_{\text{stoch}}$ is provably efficient, in the sense that its CPR is sublinear in $T$ and polynomial in $m, \numact$. However, a standard argument~\citep{slivkins19} shows that the dependency on $T$ for $\mathbb{ALG}_{\text{stoch}}$ scales as $\bigtheta \left( T^{2/3} \right)$, and it is unclear whether this dependency is optimal for tallying bandits. In the forthcoming section, we show that this dependency on $T$ can be significantly improved.

\section{Main Results}
\label{sec:results}
We now turn to our main results. In Section~\ref{sec:results_upper_bound}, we formulate Algorithm~\ref{alg:main}, a method designed for solving tallying bandit problems, and prove an $\bigo(m \numact \sqrt{T})$ upper bound on its CPR, where $\bigo$ hides only logarithmic factors. In Section~\ref{sec:results_lower_bound}, we prove an $\bigomega(\sqrt{m \numact T})$ lower bound on the CPR of any method designed to solve tallying bandit problems. This shows that Algorithm~\ref{alg:main} is nearly optimal.

\subsection{Upper Bound}
\label{sec:results_upper_bound}
To formulate our algorithm for tallying bandits, it is natural to exploit the $h$-tallying structure of the problem, by building estimates of $h_x(y)$ for each $(x,y) \in \ActSet \times \{ 1, 2 \dots m \}$. As discussed in Section~\ref{sec:tallying_bandits}, ``explore then exploit'' algorithms such as $\mathbb{ALG}_{\text{stoch}}$ incur a poor dependency on $T$. Instead, it is critical to balance exploration and exploitation, by estimating $h_x(y)$ only by playing those actions that will not increase the regret too fast. To this end, we introduce a key definition. Recall that a deterministic policy $\pi$ is length $T$ sequence of actions. We say that a deterministic policy $\pi$ is $\sqrt{T}$-cyclic if $\pi_{k \sqrt{T} + t} = \pi_t$ for each $1 \leq t \leq \sqrt{T}$ and each $0 \leq k \leq \sqrt{T} - 1$. The basis for our algorithm relies on the key claim that for any tallying bandit problem, there exists a $\sqrt{T}$-cyclic policy that is nearly optimal. This claim is formalized as Lemma~\ref{lem:cyclic_approx_well} in Appendix~\ref{app:upper}.

Assuming this claim to be true, to solve tallying bandits it is tempting to leverage algorithms designed for multi armed bandits with expert advice~\citep{auer02}, where we treat each $\sqrt{T}$-cyclic policy as an expert. However, such an approach would only guarantee $\bigo \left( T^{3/4} \right)$ CPR, since there are $\numact^{\sqrt{T}}$ experts. Instead, our algorithm draws inspiration from the successive elimination (SE) algorithm, which has been applied to sMAB~\citep{even-dar02, slivkins19}. Before we apply SE in our setting, let us recall SE in the context of sMAB. The method proceeds in epochs. Within each epoch $s$, it maintains a set $A_s$ of feasible arms, where an arm is feasible only if its estimated optimality gap lies within a confidence interval of size $C_{s-1}$. The algorithm pulls each arm in $A_s$ repeatedly to obtain a sharper estimate of its optimality gap. Then the method uses these sharper estimates to prune $A_s$ and create a smaller set $A_{s+1}$, where $A_{s+1}$ contains only those arms in $A_s$ whose estimated optimality gap is smaller than some $C_s < C_{s-1}$.

To apply SE in our tallying bandit setting, we define the initial set $A_1$ as the set of all $\sqrt{T}$-cyclic policies, and treat each such policy to be analogous to an ``arm'' in sMAB. Our aforementioned key claim ensures that there is some policy in $A_1$ that is guaranteed to be nearly optimal. However, there are two salient technical issues that prevent a naive application of SE to solve tallying bandits. First, note that the cardinality of $A_1$ is $\numact^{\sqrt{T}}$. This is too large to apply a traditional SE approach, since naively estimating the optimality gap of each policy in $A_s$ by repeatedly playing the policy and applying a concentration inequality would incur large regret. To resolve this, we exploit the $h$-tallying structure of the problem to modify SE, and estimate the optimality gap of each policy in $A_s$ by iteratively estimating $h_x(y)$ for each $(x,y) \in \ActSet \times \{ 1, 2 \dots m \}$. Second, note that unlike in the sMAB, in the tallying bandit the prior history of actions affects the loss of the current action, which biases the estimation of the optimality gap of the policies in $A_s$. To handle this, we modify SE to incorporate an additional overheard step before the estimation, and our proof shows that this overhead step removes the bias from the estimation without incurring much additional regret.

\begin{algorithm}[hbt!]
\caption{Successive Elimination for Tallying Bandits (SE-TB)}
\label{alg:main}
\begin{algorithmic}[1]
%\STATE \text{Inputs: horizon length $H$, distribution $\mathcal{D}$, sample size $n$, oracle $\oracle$ as defined in WIO}
%\State Initialize $\mathcal{S}_\pi$ and $\mathcal{A}_\pi$ each as the empty set
%\State Initialize $\pi$ as the empty function from $\mathcal{S}_\pi$ to $\mathcal{A}_\pi$
\Require memory capacity $m$, time horizon $T$, failure probability tolerance $\delta \in (0,1)$, number of actions $\numact$
\State Define $S = \log_2 \left( \frac{\sqrt{T}}{4 \numact m} + 1 \right)$.
\State Define $n_s = 2^s$, $T_s = 2 n_s \numact m \sqrt{T}$ and $C_s = \sqrt{ \frac{32 \numact m}{n_s \sqrt{T}} \log \left( \frac{2 \numact mS}{\delta} \right) }$.
\State Construct $A_1$ to be the set of all $\sqrt{T}$-cyclic policies.
\For{$s \in \{ 1, 2 \dots S \}$}
	\For{$x \in \ActSet$}
		\For{$y \in \{ 1, 2 \dots m \}$}
			\State Select $\pi_{sxy} \in \argmax_{ \pi' \in A_s } \{ N_{xy}(\pi') \}$, where $N_{xy}$ is defined in Eq.~\eqref{eq:nxy_main}.
			\If{$N_{xy}(\pi_{sxy}) = 0$}
				\State Execute $\pi_{sxy}$ for $2n_s$ periods and store nothing.
			\Else
				\State Execute $\pi_{sxy}$ for $n_s$ periods and store nothing.
				\State Execute $\pi_{sxy}$ for $n_s$ periods and store $\{ \widetilde{h}_x(y)_{s,k} \}_{k=1}^{n_s N_{xy}(\pi_{sxy}) \sqrt{T}}$.
			\EndIf
	\EndFor
	\EndFor
	\For{$\pi \in A_s$}
	\State Define $\widehat{\mu}_s(\pi) = \sum_{(x,y) \in \ActSet \times \{ 1, 2 \dots m \}} N_{xy}(\pi) \frac{1}{n_s N_{xy}(\pi_{sxy}) \sqrt{T}} \sum_{k=1}^{n_s N_{xy}(\pi_{sxy}) \sqrt{T}} \widetilde{h}_x(y)_{s,k}$.
	\EndFor
\State Select $\widehat{\pi}_s \in \argmin_{\pi \in A_s} \widehat{\mu}_s(\pi)$.
\State Construct $A_{s+1} = \left \{ \pi \in A_{s} \text{ s.t. } \widehat{\mu}_s(\pi) \leq \widehat{\mu}_s(\widehat{\pi}_s) + 2 C_s \right \}$.
\EndFor
\end{algorithmic}
\end{algorithm}

With this outline in mind, let us present our method, which is formalized in Algorithm~\ref{alg:main}. To define Algorithm~\ref{alg:main}, we say that to \emph{execute} a $\sqrt{T}$-cyclic policy $\pi$ for $k \leq \sqrt{T}$ periods means to choose the action sequence $\pi_1, \pi_2 \dots \pi_{k \sqrt{T}}$. We also define for each $\sqrt{T}$-cyclic policy $\pi$ and $(x, y) \in \ActSet \times \{ 1, 2 \dots m \}$, the quantity $N_{xy}(\pi)$ via the following procedure. Execute $\pi$ for $n+1 \leq \sqrt{T}$ periods so that we have played the action sequence $\pi_1 \dots \pi_{n \sqrt{T}}, \pi_{n \sqrt{T} + 1} \dots \pi_{(n+1)\sqrt{T}}$. Then use this action sequence to define
\begin{equation}
\label{eq:nxy_main}
N_{xy}(\pi) = \frac{1}{\sqrt{T}} \sum_{t=n \sqrt{T} + 1}^{(n+1)\sqrt{T}} \mathbb{I}(\pi_t = x) \cdot \mathbb{I} \left( y = \sum_{t'=\max \{ 1, t-m+1 \}}^t \mathbb{I}(\pi_{t'} = x) \right).
\end{equation}
Intuitively, $N_{xy}(\pi)$ is the fraction of times that the player (stochastically) observes the loss value $h_x(y)$ when they repeatedly play the $\sqrt{T}$-cyclic policy $\pi$. In Lemma~\ref{lem:N_welldef} in Appendix~\ref{app:upper}, we show that if $m \leq \sqrt{T}$, then as long as $n \geq 1$, the number $N_{xy}(\pi)$ is well defined and independent of $n$, and also independent of any action sequence that was played before we executed $\pi$ for $n+1$ periods. It suffices to consider tallying bandit problems where $m \leq \sqrt{T}$ (as we show in our proofs). Hence, lines 11 and 12 in Algorithm~\ref{alg:main} are well defined, because when we execute $\pi$ for $2 n_s \geq 2$ periods, then in the latter $n_s \geq 1$ periods we observe (stochastic instantiations of) the loss value $h_x(y)$ for a total of $n_s N_{xy}(\pi)\sqrt{T}$ times, and we have denoted these observations as $\{ \widetilde{h}_x(y)_{s,k} \}_{k=1}^{n_s N_{xy}(\pi_{sxy}) \sqrt{T}}$. Note that various steps in Algorithm~\ref{alg:main}, such as line 7, require knowledge of $N_{xy}(\pi)$, but this can be computed offline when $m$ is known. Indeed, the only steps of Algorithm~\ref{alg:main} that are online (or incur regret) are lines 9, 11 and 12. Let us now analyze the performance of Algorithm~\ref{alg:main}.

\begin{dhruvthm}
\label{thm:upper}
For any $(m,g,h)$-tallying bandit problem and any input $\delta \in (0,1)$, with probability at least $1-\delta$ the complete policy regret of Algorithm~\ref{alg:main} (SE-TB) is upper bounded as
$$
\cpReg \leq 1200 \numact m \sqrt{T} \left( \sqrt{ \log \left( 2 \numact m \log(T) / \delta \right) } + \log_2 \left( \sqrt{T} / (2 \numact m) \right) \right).
$$
\end{dhruvthm}
The proof of Theorem~\ref{thm:upper} is deferred to Appendix~\ref{app:upper}. This result guarantees that given any tallying bandit problem, Algorithm~\ref{alg:main} efficiently minimizes CPR, with a favorable dependency on $m, \numact, T$. Nevertheless, we acknowledge that our result has the following two limitations. \\

\noindent \textbf{Knowledge of $\bf m$.} Algorithm~\ref{alg:main} requires $m$ as an input, which is unrealistic in practice. It is possible to modify Algorithm~\ref{alg:main} to be adaptive to an unknown $m$, albeit at the expense of polynomially worse (and not sharp) dependency on $m, \numact$. Our focus is on obtaining a sharp characterization of the achievable CPR when $m$ is known, and so we relegate discussion of this modification to Section~\ref{sec:discussion}. Sharply characterizing the minimax CPR when $m$ is unknown remains an important open question. \\

\noindent \textbf{Computational Efficiency.} Algorithm~\ref{alg:main} is computationally inefficient. We emphasize that our exclusive focus is on statistical (rather than computational) efficiency, since our work is only a first step. We believe this is a worthwhile endeavor, since attaining sublinear CPR is a non-trivial task riddled with subtleties, even in settings that make much stronger assumptions than we do. For instance, the improving~\citep{heidari16} and single peaked~\citep{lindner21} bandit settings enforce $m=T$, require monotonicity and convexity conditions on $\{ h_x \}_{x \in \ActSet}$, and also require the losses are observed \emph{deterministically}. Even with these strong requirements, the best known CPR guarantees are \emph{asymptotic} bounds that may decay \emph{arbitrarily} slowly, and their algorithms cannot handle $m < T$. Hence we believe that our effort to provide nearly optimal non-asymptotic bounds on the CPR, in our realistic and practically motivated setting where $m \leq T$ and losses are observed stochastically, is worthwhile. Nevertheless, devising computationally efficient algorithms for the tallying bandit remains an important future direction, and we believe this is a non-trivial task. Indeed, even for the congested bandit~\citep{awasthi22}, which is the tallying bandit with the additional strong assumption that $\{ h_x \}_{x \in \ActSet}$ are increasing, existing algorithms are computationally inefficient.

\subsection{Lower Bound}
\label{sec:results_lower_bound}
It is reasonable to question whether the dependency of Algorithm~\ref{alg:main} on $m, \numact, T$ is optimal. Since the tallying bandit is equivalent to sMAB when $m=1$, a classical result~\citep{slivkins19} shows that any tallying bandit algorithm suffers $\bigomega \left( \sqrt{\numact T} \right)$ expected CPR. However, the correct dependency on $m$ is unclear when $m > 1$, due to the highly structured nature of the tallying bandit. For instance, the proof of Theorem~\ref{thm:upper} shows that any tallying bandit problem can be equivalently cast as a Markov decision process (MDP), where it takes at most $m$ timesteps to transition from any state to any other state in this MDP. One may wonder whether we can utilize such structure to design a smarter algorithm that exchanges the \emph{multiplicative} dependence on $m$ in the result of Theorem~\ref{thm:upper} for an \emph{additive} dependence on $m$. Concretely, one may desire a bound that scales as $\bigo \left( \text{poly}(m, \numact) + \numact \sqrt{T} \right)$. The following result shows that this is impossible.

\begin{dhruvthm}
\label{thm:lower}
There exists an $(m,g,h)$-tallying bandit problem and a numerical constant $c > 0$, such that the (expected) complete policy regret of any (possibly randomized) algorithm on this problem is lower bounded as
$$
\E \left[ \cpReg \right] \geq c \cdot \max \left \{ m \numact, \sqrt{m \numact T} \right \}.
$$
\end{dhruvthm}
The proof of Theorem~\ref{thm:lower} is deferred to Appendix~\ref{app:lower_bound_proof}. This result demonstrates that Algorithm~\ref{alg:main} is nearly minimax optimal, and its suboptimality is bounded by $\bigo \left( \sqrt{m \numact} \log \left( T, m, \numact \right) \right)$. A comment on the proof technique of Theorem~\ref{thm:lower} is in order. Our proof reduces the tallying bandit setting to that of best arm identification in sMAB problems~\citep{audibert10, slivkins19}. We construct a tallying bandit problem where minimizing CPR is at least as hard as identifying the best arm in an sMAB problem with $\bigtheta \left( m \numact \right)$ arms. Indeed, when $m=1$ then tallying bandit is equivalent to sMAB, and Theorem~\ref{thm:lower} recovers the classical lower bound on the expected regret suffered by any algorithm designed for sMAB with $\numact$ arms. Notably, a key component of our proof is to non-trivially upper bound the cumulative loss of the optimal policy in our construction, to ensure that we get a multiplicative (in lieu of additive) dependence on $m$ in our lower bound.

\section{Related Work}
\label{sec:related_work}
\noindent \textbf{Policy Regret.} The incompatibility of the traditional regret with an adaptive adversary was first identified by Merhav et al.~\citep{mehrav02}, who studied the full information feedback model. The notion of policy regret was formalized by the foundational work of Arora et al.~\citep{arora12}. They provide an algorithm which efficiently minimizes constant action policy regret $\pRegconst$ against generic $m$-memory bounded adversaries. It is unclear how to apply this algorithm to our setting, since our focus is on minimizing the CPR $\cpReg$. Arora et al.~\citep{arora12} do also consider minimizing the policy regret $\pReg$ when $\compclass \supsetneq \compconstant$. For instance, when $\compclass$ is the set of all piecewise constant sequences with at most $s$ switches, they provide an algorithm whose policy regret satisfies $\pReg \leq \bigo \left( m (\numact s)^{1/3} T^{2/3} \right)$. Our Counterexample~\ref{counterexample:not_constant_action} shows that this algorithm cannot minimize CPR in the tallying bandit setting, since the optimal policy in Counterexample~\ref{counterexample:not_constant_action} has $\bigtheta(T)$ switches. Cesa-Bianchi et al.~\citep{cesa-bianchi13}, Dekel et al.~\citep{dekel14} and Arora et al.~\citep{arora18} study the constant action policy regret, and do not discuss CPR. The results of Mohri and Yang~\citep{mohri18} can be extended to yield policy regret guarantees relative to rather large comparator classes, but do not provide CPR guarantees. \\

\noindent \textbf{Reinforcement Learning (RL).} Minimizing CPR is equivalent to the performance metric used in RL, which is to maximize the total collected reward~\citep{suttonbarto98}. As we show (see Lemma~\ref{lem:bandit_mdp}), any tallying bandit problem can be cast as an RL problem. However, RL methods typically scale with the cardinality of the state space~\citep{azar13, azar17, jin18}. Hence, applying off the shelf RL algorithms to solve tallying bandit problems would incur $\bigomega \left( \numact^m \right)$ CPR. \\

\noindent \textbf{Restless \& Rested Bandits.} In restless bandits, the reward of an arm evolves according to a stochastic process, independently of the actions chosen by the player~\citep{whittle81, moulines11, besbes14}. This is incompatible with tallying bandits. By contrast, tallying bandits is a special case of rested bandits, which is a general non-stationary MAB framework where an arm's reward changes when it is pulled. Tekin and Liu~\citep{tekin12} and Cortes et al.~\citep{cortes20} both study rested bandits where an arm's reward evolves according to a stochastic process, but both consider notions of regret that are significantly weaker than the CPR. A different variant of rested bandits is studied by Bouneffouf and F\'eraud~\citep{bouneffouf16}, who assume that the dynamics of how the reward changes is known upto a constant factor, and hence their results are incomparable to ours. \\

\noindent \textbf{Recharging, Recovering, Blocking, Delay-Dependent \& Last Switch Dependent Bandits.} This line of work studies settings where an arm's reward changes according to the number of timesteps that have passed since the arm was last pulled~\citep{immorlica18, basu19, pike-burke19, cella20, laforgue21}. Such settings typically cannot be cleanly classified as either rested or restless bandits, but are related to both. The models in these works for how the reward evolves are different than our tallying structure, where the reward of an arm instead depends on the number of times that arm was played. \\

\noindent \textbf{Rotting Bandits.} The rotting bandits setting~\citep{heidari16, levine17, seznec19, seznec20} is a special case of our tallying bandit formulation, and merits close comparison to our work. This setting enforces $m=T$, and an arm's reward is assumed to be a decreasing function of the number of times that arm has been pulled. In our language, this means the $\{ h_x \}_{x \in \ActSet}$ functions are increasing. This strong assumption enables these works to efficiently minimize CPR even though $m=T$, although we note that algorithms for this setting cannot handle general $m < T$. By contrast, in our setting we make no assumptions on the structure of the functions $\{ h_x \}_{x \in \ActSet}$, and assuming $m < T$ is necessary. Indeed, in our setting, when $m$ is $\bigomega(T)$ then the result of Theorem~\ref{thm:lower} shows that any player suffers $\bigomega(T)$ worst case CPR. \\

\noindent \textbf{Improving \& Single Peaked Bandits.} The improving bandit~\citep{heidari16} and single peaked bandit~\citep{lindner21} are both special cases of our tallying bandit setting, and deserve special attention. In improving bandits, the reward of an arm is an increasing, concave function of the number of times it has been pulled. The single peaked bandit generalizes this, so that the reward function of an arm is initially increasing and concave, but may become decreasing at some point. In our language, this means that the $\{ h_x \}_{x \in \ActSet}$ functions are decreasing and convex, or decreasing and convex and then possibly increasing after some point. Both settings enforce $m=T$, and algorithms for these settings do not handle general $m < T$. By contrast, since we make no assumptions on the structure of the functions $\{ h_x \}_{x \in \ActSet}$, our Theorem~\ref{thm:lower} shows that when $m$ is $\bigomega(T)$ then the worst case CPR scales as $\bigomega(T)$. We remark that the CPR bounds in these works are asymptotic, whereas we provide non-asymptotic guarantees. We also remark that these works require the losses to be observed deterministically, and they only provide a heuristic to handle stochastic observations of the loss. \\

\noindent \textbf{Congested Bandits.} In concurrent work, Awasthi et al.~\citep{awasthi22} introduced the congested bandit, which is a special case of the tallying bandit. Their formulation considers arbitrary $m \leq T$, and the reward of an arm is a decreasing function of the number of times it has been pulled. Hence the congested bandit is the tallying bandit with the additional assumption that the $\{ h_x \}_{x \in \ActSet}$ functions are increasing. This additional assumption enables them to provide an $\bigo \left( \sqrt{mKT} \right)$ CPR bound, although we note that their algorithm is computationally inefficient.

\section{Discussion}
\label{sec:discussion}
\noindent In this paper, we studied conditions under which it is possible to efficiently minimize CPR in online learning. To this end, it is necessary the restrict the adversary, and we considered several natural restrictions on the adversary that have appeared in prior work. We then exposed a gap in our understanding of when it is possible to efficiently minimize CPR. To resolve this gap, we introduced the tallying bandit setting, and formulated an algorithm whose CPR (after discarding logarithmic factors) is w.h.p at most $\bigo \left( m \numact \sqrt{T} \right)$. We also provided a lower bound of $\bigomega \left( \sqrt{m \numact T} \right)$ on the expected CPR of any tallying bandit algorithm, demonstrating the near optimality of our method.

Our Algorithm~\ref{alg:main} required as input the true value of $m$. In practice, this knowledge is unrealistic, and one instead might only have an upper bound $\overline{m}$ on the true value of $m$. Let us describe a modified version of Algorithm~\ref{alg:main} that can be used in this setting. Recall from the proof of Theorem~\ref{thm:upper} that Algorithm~\ref{alg:main} proceeds in epochs, and in each epoch $s$ it stores a set $A_s$ of policies whose (average) loss is $\bigo \left( 2^{-s} T^{-1/4} \right)$ greater than that of the optimal policy. Hence, for the setting with unknown $m$, we can run $\overline{m}$ instantiations of Algorithm~\ref{alg:main}. After each epoch $s$, we identify the instantiation that has the policy with the minimum estimated (average) loss, and denote this instantiation as $m_s$. We then discard those instantiations whose policies have (average) loss that is $\bigomega \left( 2^{-s} T^{-1/4} \right)$ greater than the (average) loss of the best policy stored by $m_s$. With such an approach, we are guaranteed to never discard the instantiation corresponding to the true $m$. And via the techniques used in the proof of Theorem~\ref{thm:upper}, we can show that if we do not discard an instantiation corresponding to $m' \neq m$, then playing policies stored by this instantiation does not incur large regret. This approach yields a $\bigo \left( \sqrt{T} \right)$ upper bound on the CPR, at the expense of polynomial factors of $\overline{m}, \numact$.

A number of open directions remain. A natural open question is resolving the gap between our upper and lower bounds on the achievable CPR in the tallying bandit. Separately, although Algorithm~\ref{alg:main} is nearly statistically optimal, it is computationally inefficient. However, since the tallying bandit is highly structured, it is possible that the computational efficiency of even our own Algorithm~\ref{alg:main} can be improved. For instance, can we design a data structure, which stores policies in a manner that allows efficient elimination of suboptimal policies after each epoch? Devising computationally efficient algorithms for tallying bandits is an important direction for future work. Finally, we view the tallying bandit as only a first step towards modeling interactive settings like recommender systems. For example, consider the following generalization of tallying bandits, where the loss of an action is a function of a \emph{weighted} sum of the number of times the action has been played in the past, where more recent plays are given more weight. This naturally corresponds to a model of human memory, where more importance is placed on more recent events. Can we design efficient algorithms for such settings?

% Acknowledgments---Will not appear in anonymized version
\subsection*{Acknowledgements}
This material is based upon work supported by the National Science Foundation Graduate Research Fellowship Program under Grant No. DGE1745016. Any opinions, findings, and conclusions or recommendations expressed in this material are those of the authors and do not necessarily reflect the views of the National Science Foundation.

\appendix

\section{Analysis of Algorithm~\ref{alg:main}}
\label{app:upper}

\noindent In this section, we analyze the complete policy regret of Algorithm~\ref{alg:main}, and prove Theorem~\ref{thm:upper}. Before we formally prove Theorem~\ref{thm:upper}, we first state below two key lemmas, that will be useful throughout. The first lemma shows an equivalence between tallying bandit problems and Markov decision processes (MDPs)~\citep{suttonbarto98}. The second lemma verifies that the $N_{xy}(\pi)$ quantity defined in Algorithm~\ref{alg:main} is well defined. We additionally introduce new quantities $\mu$ and $\pi^\star$ which will be useful for our proofs. With this outline in mind, let us begin the analysis.

\begin{restatable}{dhruvlemma}{mdplem}
\label{lem:bandit_mdp}
Any $(m,g,h)$-tallying bandit problem can be equivalently expressed as a finite horizon Markov decision process (MDP).
\end{restatable}

\noindent The proof of this Lemma~\ref{lem:bandit_mdp} is given in Appendix~\ref{app_proof:lem:bandit_mdp}. For the sake of brevity, we have not stated the explicit details of the reduction (for instance, the definition of state space or transition function of the corresponding MDP) in the statement of this lemma. Nevertheless, these details are readily found in the proof. \\

\noindent Recall that in Section~\ref{sec:results_upper_bound}, to facilitate the definition of Algorithm~\ref{alg:main} we defined the quantity $N_{xy}(\pi)$ via the following procedure. Execute $\pi$ for $n+1$ periods so that we have played the action sequence $\pi_1, \pi_2 \dots \pi_{n \sqrt{T}}, \pi_{n\sqrt{T} + 1} \dots \pi_{(n+1)\sqrt{T}}$. Then use this action sequence to define
$$
N_{xy}(\pi) = \frac{1}{\sqrt{T}} \sum_{t=n \sqrt{T} + 1}^{(n+1)\sqrt{T}} \mathbb{I}(\pi_t = x) \cdot \mathbb{I} \left( y = \sum_{t'=\max \{ 1, t-m+1 \}}^t \mathbb{I}(\pi_{t'} = x) \right).
$$
The next lemma shows that $N_{xy}(\pi)$ is always well defined if $1 \leq n \leq \sqrt{T} - 1$ and $m \leq \sqrt{T}$.

\begin{dhruvlemma}
\label{lem:N_welldef}
Consider any $(m,g,h)$-tallying bandit problem where $m \leq \sqrt{T}$, and fix any $\sqrt{T}$-cyclic policy $\pi$, any $x \in \ActSet$ and any $y \in \{ 1,2 \dots m \}$. When defined via the aforementioned procedure, the quantity $N_{xy}(\pi)$ is well defined for any $1 \leq n \leq \sqrt{T} - 1$. Furthermore, $N_{xy}(\pi)$ is independent of any action sequence that was played before $\pi$ was executed for $n+1$ periods.
\end{dhruvlemma}

\noindent The proof of this Lemma~\ref{lem:N_welldef} is deferred to Appendix~\ref{app_proof:lem:N_welldef}. Now for each $\sqrt{T}$-cyclic policy $\pi$, we define the quantity $\mu(\pi)$ as follows
$$
\mu(\pi) = \sum_{(x,y) \in \ActSet \times \{ 1, 2 \dots m \}} N_{xy}(\pi) h_x(y).
$$
Via this definition and the result of Lemma~\ref{lem:N_welldef}, it is immediate that when $m \leq \sqrt{T}$, if we play an arbitrary action sequence and then execute $\pi$ for $n+1$ periods, then the (expected) cumulative loss experienced in the final period (i.e., the $(n+1)$th period) is $\mu(\pi)\sqrt{T}$. We use this notion of $\mu$ to define the policy $\pi^\star$ as
$$
\pi^\star \in \argmin_{\pi \in A_1} \mu(\pi).
$$
Recall for this definition that $A_1$ as defined in Algorithm~\ref{alg:main} is the set of all $\sqrt{T}$-cyclic policies. With these definitions in hand, we are now in a position to formally prove Theorem~\ref{thm:upper}.

\subsection{Proof of Theorem~\ref{thm:upper}}

\noindent First, we note that if $m > \sqrt{T}$, then the statement of the theorem is trivially true since the complete policy regret is always upper bounded by $T$. Hence, for the remainder of the proof it suffices to assume that $m \leq \sqrt{T}$. For any policy $\pi$, which is a length $T$ deterministic sequence of actions, let $\ell_t(\pi)$ denote the expected loss suffered at timestep $t$ while playing $\pi$. Our next lemma shows that for any tallying bandit problem, the loss suffered by the optimal policy (i.e., the policy in $\ActSet^T$ that experiences the minimum cumulative expected loss) can be well approximated by $T \mu(\pi^\star)$.

\begin{dhruvlemma}
\label{lem:cyclic_approx_well}
Given any $(m,g,h)$-tallying bandit problem, let $\pi^{\star \star}$ denote an optimal policy in $\ActSet^T$. Then
$$
T \mu(\pi^\star) - \sum_{t=1}^T \ell_t(\pi^{\star \star}) \leq (m+1) \sqrt{T}.
$$
\end{dhruvlemma}

\noindent The proof of Lemma~\ref{lem:cyclic_approx_well} is deferred to Appendix~\ref{app_proof:lem:cyclic_approx_well}. Now let $\ell^s$ denote the loss experienced in epoch $s \in \{ 1, 2 \dots S \}$ of Algorithm~\ref{alg:main}. The following lemma bounds the cumulative loss of Algorithm~\ref{alg:main} relative to $T \mu(\pi^\star)$.

\begin{dhruvlemma}
\label{lem:main_regret_helper}
Assume that $m \leq \sqrt{T}$. With probability at least $1 - \delta$, the total loss of Algorithm~\ref{alg:main} relative to $T \mu(\pi^\star)$ can be upper bounded as
$$
\sum_{s=1}^S \ell^s - T \mu(\pi^\star) \leq \numact m \sqrt{T} \left( 5 \log_2 \left( \frac{\sqrt{T}}{4 \numact m} + 1 \right) + 400 \sqrt{ \log \left( \frac{2 \numact m \log(T)}{\delta} \right) } \right).
$$
\end{dhruvlemma}
The proof of this Lemma~\ref{lem:main_regret_helper} is provided in Appendix~\ref{app_proof:lem:main_regret_helper}. With the results of Lemma~\ref{lem:cyclic_approx_well} and Lemma~\ref{lem:main_regret_helper} in hand, we now utilize them to prove Theorem~\ref{thm:upper} as follows. Note that the complete policy regret $\cpReg$ of Algorithm~\ref{alg:main} satisfies
\begin{align*}
\cpReg &= \sum_{s=1}^S \ell^s - \sum_{t=1}^T \ell_t(\pi^{\star \star}) \\
&= \sum_{s=1}^S \ell^s - T \mu(\pi^\star) + T \mu(\pi^\star) - \sum_{t=1}^T \ell_t(\pi^{\star \star}) \\
&\leq \numact m \sqrt{T} \left( 5 \log_2 \left( \frac{\sqrt{T}}{4 \numact m} + 1 \right) + 400 \sqrt{ \log \left( \frac{2 \numact m \log(T)}{\delta} \right) } \right) + (m+1) \sqrt{T} \\
&\leq 1200 \numact m \sqrt{T} \left( \sqrt{ \log \left( 2 \numact m \log(T) / \delta \right) } + \log_2 \left( \sqrt{T} / (2 \numact m) \right) \right).
\end{align*}
This completes the proof of Theorem~\ref{thm:upper}. \hfill \qeddhruv

\subsection{Proof of Lemma~\ref{lem:main_regret_helper}}
\label{app_proof:lem:main_regret_helper}
\noindent To facilitate the proof, we require the following critical lemma, which bounds the loss incurred by Algorithm~\ref{alg:main} in each epoch $s \in \{ 1, 2 \dots S\}$. For the statement of the following lemma, note that completing any epoch $s \in \{ 1, 2 \dots S \}$ takes a total of $T_s = 2 n_s \numact m \sqrt{T}$ timesteps.
\begin{dhruvlemma}
\label{lem:loss_each_s}
Assume that $m \leq \sqrt{T}$. With probability at least $1 - \delta$, we have simultaneously for each epoch $s \in \{ 2, 3 \dots S \}$ that the total loss relative to $T_s \mu(\pi^\star)$ is bounded as
$$
\ell^s - T_s \mu(\pi^\star) \leq \numact m \left( \sqrt{T} + 8 n_s \sqrt{T} C_{s-1} \right).
$$
\end{dhruvlemma}

\noindent The proof of this Lemma~\ref{lem:loss_each_s} is provided in Appendix~\ref{app_proof:lem:loss_each_s}. Observe that by the result of Lemma~\ref{lem:loss_each_s}, we are guaranteed with probability at least $1 - \delta$ that
\begin{equation}
\label{eqn:help4}
\begin{aligned}
\sum_{s=1}^S \ell^s - T \mu(\pi^\star) &= \sum_{s=1}^S \left( \ell^s - T_s \mu(\pi^\star) \right) \\
&\leq 4 \numact m \sqrt{T} + \sum_{s=2}^S \left( \ell^s - T_s \mu(\pi^\star) \right) \\
&\leq 4 \numact m \sqrt{T} + \sum_{s=2}^S \numact m \left( \sqrt{T} + 8 n_s \sqrt{T} C_{s-1} \right) \\
&\leq 5 S \numact m \sqrt{T} + 8 \numact m \sum_{s=2}^S n_s \sqrt{T} C_{s-1}.
\end{aligned}
\end{equation}
Now substituting in the definitions
$$
n_s = 2^s \text{ and } C_{s-1} = \sqrt{ \frac{32 \numact m}{n_{s-1} \sqrt{T}} \log \left( \frac{2 \numact mS}{\delta} \right) } = \sqrt{ \frac{64 \numact m}{n_s \sqrt{T}} \log \left( \frac{2 \numact mS}{\delta} \right) },
$$
which were provided in Algorithm~\ref{alg:main}, into the final term on the RHS of Eq.~\eqref{eqn:help4} yields that
\begin{align*}
8 \numact m \sum_{s=2}^S n_s \sqrt{T} C_{s-1} &= 8 \numact m \sum_{s=2}^S n_s \sqrt{T} \sqrt{ \frac{64 \numact m}{n_s \sqrt{T}} \log \left( \frac{2 \numact mS}{\delta} \right) } \\
&= 64 \numact^{1.5} m^{1.5} \sqrt{ \log \left( \frac{2 \numact mS}{\delta} \right) } \sum_{s=2}^S n_s \sqrt{T} \sqrt{ \frac{1}{n_s \sqrt{T}} } \\
&= 64 \numact^{1.5} m^{1.5} \sqrt{ \log \left( \frac{2 \numact mS}{\delta} \right) } T^{1/4} \sum_{s=2}^S \sqrt{n_s} \\
&= 64 \numact^{1.5} m^{1.5} \sqrt{ \log \left( \frac{2 \numact mS}{\delta} \right) } T^{1/4} \sum_{s=2}^S 2^{s/2} \\
&\leq 400 \numact^{1.5} m^{1.5} \sqrt{ \log \left( \frac{2 \numact mS}{\delta} \right) } T^{1/4} 2^{S/2}.
\end{align*}
Finally, we recall the definition of $S = \log_2 \left( \frac{\sqrt{T}}{4 \numact m} + 1 \right)$ to observe that
\begin{equation}
\label{eqn:help5}
\begin{aligned}
8 \numact m \sum_{s=2}^S n_s \sqrt{T} C_{s-1} &\leq 400 \numact^{1.5} m^{1.5} \sqrt{ \log \left( \frac{2 \numact mS}{\delta} \right) } T^{1/4} 2^{S/2} \\
&= 400 \numact^{1.5} m^{1.5} \sqrt{ \log \left( \frac{2 \numact mS}{\delta} \right) } T^{1/4} \sqrt{ \left( \frac{\sqrt{T}}{4 \numact m} + 1 \right) } \\
&\leq 400 \numact^{1.5} m^{1.5} \sqrt{ \log \left( \frac{2 \numact mS}{\delta} \right) } T^{1/4} \frac{T^{1/4}}{\sqrt{\numact m} } \\
&= 400 \numact m \sqrt{ \log \left( \frac{2 \numact mS}{\delta} \right) } \sqrt{T} \\
&\leq 400 \numact m \sqrt{ \log \left( \frac{2 \numact m \log(T)}{\delta} \right) } \sqrt{T}
\end{aligned}
\end{equation}
Combining Eq.~\eqref{eqn:help4} with Eq.~\eqref{eqn:help5} yields the result. \hfill \qeddhruv

\subsection{Proof of Lemma~\ref{lem:loss_each_s}}
\label{app_proof:lem:loss_each_s}
\noindent To facilitate the proof, we leverage the following critical lemma, which bounds the gap of the average value $\mu$ of policies in $A_s$ versus $\mu(\pi^\star)$.

\begin{dhruvlemma}
\label{lem:opt_gap}
Assume that $m \leq \sqrt{T}$. The event
$$
\cap_{s=1}^S \cap_{\pi \in A_s} \left \{ \mu(\pi) - \mu(\pi^\star) \leq 4 C_{s-1} \right \},
$$
occurs with probability at least $1 - \delta$.
\end{dhruvlemma}

\noindent The proof of this Lemma~\ref{lem:opt_gap} is provided in Appendix~\ref{app_proof:lem_opt_gap}. Let us now return to the main proof. For ease in notation, let $\ell^{sxy}$ denote the total loss experienced in epoch $s$ of Algorithm~\ref{alg:main} while executing the policy $\pi_{sxy}$ for $2 n_s$ epochs. Hence we have $\ell^s = \sum_{(x, y) \in \ActSet \times \{ 1, 2 \dots m \}} \ell^{sxy}$. \\

\noindent Note that within a single epoch $s > 1$, for each $(x, y) \in \ActSet \times \{ 1, 2 \dots m \}$ we execute the policy $\pi_{sxy}$ for $2 n_s$ periods, where each period takes $\sqrt{T}$ timesteps. Recall the fact that when $m \leq \sqrt{T}$, if we play an arbitrary action sequence and then execute $\pi_{sxy}$ for $n+1$ periods for $n \geq 1$, then the (expected) cumulative loss experienced in the final period (i.e., the $(n+1)$th period) is $\mu(\pi_{sxy}) \sqrt{T}$. In particular, this fact implies that if we execute $\pi_{sxy}$ for $2 n_s$ periods, then the total loss experienced $\ell^{sxy}$ during these $2 n_s$ periods is upper bounded by
$$
\ell^{sxy} \leq \mu(\pi_{sxy}) \sqrt{T} (2n_s-1) + \sqrt{T} \leq 2 n_s \mu(\pi_{sxy}) \sqrt{T} + \sqrt{T}.
$$
Hence we can use Lemma~\ref{lem:opt_gap} to upper bound
\begin{align*}
\ell^{sxy} - 2 n_s \mu(\pi^\star) \sqrt{T} &\leq 2 n_s \mu(\pi_{sxy}) \sqrt{T} + \sqrt{T} - 2 n_s \mu(\pi^\star) \sqrt{T} \\
&= \sqrt{T} + 2n_s\sqrt{T} \left( \mu(\pi_{sxy}) - \mu(\pi^\star) \right) \\
&\leq \sqrt{T} + 8 n_s \sqrt{T} C_{s-1}.
\end{align*}
This bound holds uniformly for each $(x, y) \in \ActSet \times \{ 1, 2 \dots m \}$, and hence we have that
$$
\ell^s - T_s \mu(\pi^\star) = \sum_{(x, y) \in \ActSet \times \{ 1, 2 \dots m \}} \left( \ell^{sxy} - 2 n_s \mu(\pi^\star) \sqrt{T} \right) \leq \numact m \left( \sqrt{T} + 8 n_s \sqrt{T} C_{s-1} \right).
$$
This completes the proof. \hfill \qeddhruv

\subsection{Proof of Lemma~\ref{lem:opt_gap}}
\label{app_proof:lem_opt_gap}
To facilitate the proof, we require the following two critical helper results. The first result bounds the error incurred when estimating $\mu(\pi)$ via the stochastic realizations $\left \{ \{ \widetilde{h}_x(y)_{s,k} \}_{k=1}^{n_s N_{xy}(\pi_{sxy}) \sqrt{T}} \right \}_{(x,y) \in \ActSet \times \{ 1, 2 \dots m \}}$. The second result shows that while running Algorithm~\ref{alg:main}, which is based on successive elimination of inferior policies over epochs $s \in \{ 1, 2 \dots S \}$, at any epoch $s$ we never eliminate $\pi^\star$ from our set $A_s$ of feasible policies.

\begin{dhruvlemma}
\label{lem:mu_conc}
Assume that $m \leq \sqrt{T}$. Fix any $s \in \{ 1, 2 \dots S \}$, and let $B_s$ denote the event that for all $\pi \in A_s$ we simultaneously have that
$$
\left \vert \widehat{\mu}_s(\pi) - \mu(\pi) \right \vert \leq C_s.
$$
Then $B_s$ occurs with probability at least $1 - \delta/S$.
\end{dhruvlemma}

\begin{dhruvlemma}
\label{lem:pi_star_in_As}
Assume that $m \leq \sqrt{T}$. The event $\cap_{s=1}^S B_s$, where the event $B_s$ is defined in Lemma~\ref{lem:mu_conc}, implies the event that
$$
\pi^\star \in \cap_{s=1}^S A_s \text{ and } \cap_{s=1}^S \left \{ 0 \leq \widehat{\mu}_s(\pi^\star) - \widehat{\mu}_s(\widehat{\pi}_s) \leq 2 C_s \right \}.
$$
\end{dhruvlemma}

\noindent The proofs of Lemma~\ref{lem:mu_conc} and Lemma~\ref{lem:pi_star_in_As} are provided in Appendix~\ref{app_proof:lem:mu_conc} and Appendix~\ref{app_proof:lem:pi_star_in_As} respectively. \\

\noindent Let us now return to the proof. By the result of Lemma~\ref{lem:mu_conc} and a union bound, the event $\cap_{s=1}^S B_s$ occurs with probability at least $1- \delta$. Furthermore, the result of Lemma~\ref{lem:pi_star_in_As} shows that the event $\cap_{s=1}^S B_s$ implies the event
\begin{equation}
\label{eqn:help3}
\pi^\star \in \cap_{s=1}^S A_s.
\end{equation}
So on the event $\cap_{s=1}^S B_s$, note that for any $s$ and any $\pi \in A_s$ we have
\begin{align*}
\mu(\pi) - \mu(\pi^\star) &\overset{(i)}{\leq} \widehat{\mu}_{s-1}(\pi) - \mu(\pi^\star) + C_{s-1} \\
&\overset{(ii)}{\leq} \widehat{\mu}_{s-1}(\widehat{\pi}_{s-1}) - \mu(\pi^\star) + 3C_{s-1} \\
&\overset{(iii)}{\leq} \widehat{\mu}_{s-1}(\pi^\star) - \mu(\pi^\star) + 3C_{s-1} \\
&\overset{(iv)}{\leq} \mu(\pi^\star) - \mu(\pi^\star) + 4C_{s-1} \\
&= 4 C_{s-1},
\end{align*}
where step $(i)$ follows from Lemma~\ref{lem:mu_conc}, step $(ii)$ follows from the definition of $A_s$ and the fact that $\pi \in A_s$, step $(iii)$ follows from the definition of $\widehat{\pi}_{s-1}$ and Eq.~\eqref{eqn:help3}, and step $(iv)$ follows again from Lemma~\ref{lem:mu_conc} and Eq.~\eqref{eqn:help3}. This completes the proof. \hfill \qeddhruv

\subsection{Proof of Lemma~\ref{lem:mu_conc}}
\label{app_proof:lem:mu_conc}
\noindent For the proof of this lemma, it is useful to define the quantity $C_{sxy}$ as
$$
C_{sxy} = \sqrt{\frac{32}{n_s N_{xy}(\pi_{sxy}) \sqrt{T}} \log \left( \frac{2 \numact m S}{\delta} \right) },
$$
for each $(s,x,y) \in \{ 1, 2 \dots S \} \times \ActSet \times \{ 1, 2 \dots m \}$. Fix any $(x, y) \in \ActSet \times \{ 1, 2 \dots m \}$. Note that by Hoeffding's bound~\citep{hoeffding63}, we are guaranteed that the event
\begin{equation}
\label{eqn:helper1}
\left \vert h_x(y) - \frac{1}{n_s N_{xy}(\pi_{sxy}) \sqrt{T}} \sum_{k=1}^{n_s N_{xy}(\pi_{sxy}) \sqrt{T}} \widetilde{h}_x(y)_{s,k} \right \vert \leq C_{sxy},
\end{equation}
occurs with probability at least $1 - \delta/(\numact m S)$. A union bound then ensures that the above event occurs simultaneously for all $(x, y) \in \ActSet \times \{ 1, 2 \dots m \}$ with probability at least $1 - \delta/S$. We now claim that this (simultaneous) event is a subset of $B$, which is sufficient to complete the proof. \\

\noindent To establish the claim, note that on this event, we are guaranteed for any $\pi \in A_s$ that 
\begin{align*}
&\left \vert \mu(\pi) - \widehat{\mu}_s(\pi) \right \vert \\
= &\left \vert \sum_{(x,y) \in \ActSet \times \{ 1, 2 \dots m \}} N_{xy}(\pi) h_x(y) - \sum_{(x,y) \in \ActSet \times \{ 1, 2 \dots m \}} N_{xy}(\pi) \frac{1}{n_s N_{xy}(\pi_{sxy}) \sqrt{T}} \sum_{k=1}^{n_s N_{xy}(\pi_{sxy}) \sqrt{T}} \widetilde{h}_x(y)_{s,k} \right \vert \\
= &\left \vert \sum_{(x,y) \in \ActSet \times \{ 1, 2 \dots m \}} N_{xy}(\pi) \left( h_x(y) - \frac{1}{n_s N_{xy}(\pi_{sxy}) \sqrt{T}} \sum_{k=1}^{n_s N_{xy}(\pi_{sxy}) \sqrt{T}} \widetilde{h}_x(y)_{s,k} \right) \right \vert \\
\leq &\sum_{(x,y) \in \ActSet \times \{ 1, 2 \dots m \}} N_{xy}(\pi) C_{sxy},
\end{align*}
where the final step follows from the triangle inequality and Eq.~\eqref{eqn:helper1}. Continuing the above, we have that
\begin{align*}
\left \vert \mu(\pi) - \widehat{\mu}_s(\pi) \right \vert &\leq \sum_{(x,y) \in \ActSet \times \{ 1, 2 \dots m \}} N_{xy}(\pi) C_{sxy} \\
&= \sqrt{\frac{32}{n_s \sqrt{T}} \log \left( \frac{2 \numact m S}{\delta} \right) } \sum_{(x,y) \in \ActSet \times \{ 1, 2 \dots m \}} \frac{N_{xy}(\pi)}{\sqrt{N_{xy}(\pi_{sxy})}} \\
&\overset{(i)}{\leq} \sqrt{\frac{32}{n_s \sqrt{T}} \log \left( \frac{2 \numact m S}{\delta} \right) } \sum_{(x,y) \in \ActSet \times \{ 1, 2 \dots m \}} \sqrt{N_{xy}(\pi)} \\
&\overset{(ii)}{\leq} \sqrt{\numact m} \sqrt{\frac{32}{n_s \sqrt{T}} \log \left( \frac{2 \numact m S}{\delta} \right) } \\
&= C_s,
\end{align*}
where step $(i)$ follows from the fact that $N_{xy}(\pi_{sxy}) \geq N_{xy}(\pi)$ by its definition in Algorithm~\ref{alg:main}, and step $(ii)$ follows from the fact that $N_{xy}(\pi) \in [0,1]$, that $\sum_{(x,y) \in \ActSet \times \{ 1, 2 \dots m \}} N_{xy}(\pi) = 1$ as well as the Cauchy-Schwarz inequality. This establishes the claim and hence completes the proof. \hfill \qeddhruv

\subsection{Proof of Lemma~\ref{lem:pi_star_in_As}}
\label{app_proof:lem:pi_star_in_As}
\noindent Assume that the event $\cap_{s'=1}^S B_{s'}$ is true. On this event, we prove the lemma by induction on $s$. First we demonstrate the base case of $s=1$, which is that $\pi^\star \in A_1$ and $0 \leq \widehat{\mu}_1(\pi^\star) - \widehat{\mu}_1(\widehat{\pi}_1) \leq 2 C_1$. Then for the inductive step we show that if the event $\pi^\star \in A_{s-1}$ and $0 \leq \widehat{\mu}_{s-1}(\pi^\star) - \widehat{\mu}_{s-1}(\widehat{\pi}_{s-1}) \leq 2 C_{s-1}$ occurs, then we also have that the event 
$$
\pi^\star \in A_s \text{ and } 0 \leq \widehat{\mu}_{s}(\pi^\star) - \widehat{\mu}_{s}(\widehat{\pi}_{s}) \leq 2 C_{s},
$$
is also true. \\

\noindent For the base case, note that by definition we are guaranteed $\pi^\star \in A_1$. And by the definition of $\widehat{\pi}_1$, we know that $0 \leq \widehat{\mu}_1(\pi^\star) - \widehat{\mu}_1(\widehat{\pi}_1)$. Furthermore, recalling the definition of the event $B_1$ in Lemma~\ref{lem:mu_conc}, on the event $B_1$ we have that
$$
\widehat{\mu}_1(\pi^\star) - \mu(\pi^\star) \leq C_1 \text{ and } \mu(\widehat{\pi}_1) - \widehat{\mu}_1(\widehat{\pi}_1) \leq C_1.
$$
Putting these equations together and using the fact that $\mu(\pi^\star) \leq \mu(\widehat{\pi}_1)$ ensures that
$$
\widehat{\mu}_1(\pi^\star) - \widehat{\mu}_1(\widehat{\pi}_1) \leq 2C_1.
$$
This verifies the base case. \\

\noindent For the inductive step, assume that $\pi^\star \in A_{s-1}$ and $0 \leq \widehat{\mu}_{s-1}(\pi^\star) - \widehat{\mu}_{s-1}(\widehat{\pi}_{s-1}) \leq 2 C_{s-1}$ occurs. Then the definition of $A_s$ and the inductive hypothesis directly imply that $\pi^\star \in A_s$. Hence, it is true by definition of $\widehat{\pi}_s$ that $0 \leq \widehat{\mu}_s(\pi^\star) - \widehat{\mu}_s(\widehat{\pi}_s)$. Then recalling the definition of the event $B_s$ in Lemma~\ref{lem:mu_conc}, on the event $B_s$ we have that
$$
\widehat{\mu}_s(\pi^\star) - \mu(\pi^\star) \leq C_s \text{ and } \mu(\widehat{\pi}_s) - \widehat{\mu}_s(\widehat{\pi}_s) \leq C_s.
$$
Putting these equations together and using the fact that $\mu(\pi^\star) \leq \mu(\widehat{\pi}_s)$ ensures that
$$
\widehat{\mu}_s(\pi^\star) - \widehat{\mu}_s(\widehat{\pi}_s) \leq 2C_s.
$$
This verifies the inductive step. As argued earlier, this is sufficient to complete the proof. \hfill \qeddhruv

\subsection{Proof of Lemma~\ref{lem:cyclic_approx_well}}
\label{app_proof:lem:cyclic_approx_well}
Note that since $\mu(\pi^\star) \in [0,1]$, the statement is trivial for $m \geq \sqrt{T}$. Hence, assume for the remainder of the proof that $m < \sqrt{T}$. There exists some $k \in \{ 0, 1 \dots \sqrt{T} - 1 \}$ such that
$$
\sum_{t=k \sqrt{T} + 1}^{(k+1)\sqrt{T}} \ell_t(\pi^{\star \star}) \leq \frac{1}{\sqrt{T}} \sum_{t=1}^T \ell_t(\pi^{\star \star}).
$$
Define the $\sqrt{T}$-cyclic policy $\pi$ by letting $\pi_t = \pi^{\star \star}_{k \sqrt{T} + t}$ for all $1 \leq t \leq \sqrt{T}$. Note that via the MDP characterization provided in Lemma~\ref{lem:bandit_mdp}, we can equivalently think of the tallying bandit problem as some MDP that we denote $M$. The proof of Lemma~\ref{lem:bandit_mdp} shows that regardless of the state we start at, either playing $\pi_1, \pi_2 \dots \pi_m$ or playing $\pi^{\star \star}_{k \sqrt{T} + 1}, \pi^{\star \star}_{k \sqrt{T} + 2} \dots \pi^{\star \star}_{k \sqrt{T} + m}$ leads to the same state in $M$. From that state, either playing $\pi_{m+1}, \pi_{m+2} \dots \pi_{\sqrt{T}}$ or playing $\pi^{\star \star}_{k \sqrt{T} + m+1}, \pi^{\star \star}_{k \sqrt{T} + m+2} \dots \pi^{\star \star}_{(k+1) \sqrt{T}}$ leads to the same sequence of states, and hence the same sequence of (expected) losses. Hence we have shown that after a single first execution of $\pi$, we can bound the loss as
$$
\sum_{t=1}^{\sqrt{T}} \ell_k(\pi) \leq m + \sum_{t=k \sqrt{T} + 1}^{(k+1)\sqrt{T}} \ell_k(\pi^{\star \star}) \leq m + \frac{1}{\sqrt{T}} \sum_{t=1}^T \ell_t(\pi^{\star \star}).
$$
Repeating this argument for $\sqrt{T}$ executions of $\pi$, we have shown that
\begin{equation}
\label{eqn:help2}
\sum_{t=1}^{T} \ell_k(\pi) \leq m\sqrt{T} + \sqrt{T} \sum_{t=k \sqrt{T} + 1}^{(k+1)\sqrt{T}} \ell_k(\pi^{\star \star}) \leq m \sqrt{T} + \sum_{t=1}^T \ell_t(\pi^{\star \star}).
\end{equation}
Now, we observe that $T \mu(\pi) \leq \sum_{t=1}^{T} \ell_k(\pi) + \sqrt{T}$, where we used the fact that when $m \leq \sqrt{T}$, if we play an arbitrary action sequence and then execute $\pi$ for $n+1$ periods for $n \geq 1$, then the (expected) cumulative loss experienced in the final period (i.e., the $(n+1)$th period) is $\mu(\pi)\sqrt{T}$. Finally, we note that $\mu(\pi^\star) \leq \mu(\pi)$ by definition, and so we have that
$$
T \mu(\pi^\star) \leq T \mu(\pi) \leq \sum_{t=1}^{T} \ell_k(\pi) + \sqrt{T},
$$
which combined with Eq.~\eqref{eqn:help2} implies that
$$
T \mu(\pi^\star) - \sum_{t=1}^T \ell_t(\pi^{\star \star}) \leq (m+1) \sqrt{T}.
$$
This completes the proof. \hfill \qeddhruv

\subsection{Proof of Lemma~\ref{lem:N_welldef}}
\label{app_proof:lem:N_welldef}
To prove this result, we leverage the MDP characterization of tallying bandits provided by Lemma~\ref{lem:bandit_mdp}. Let $M$ denote the MDP corresponding to the given $(m,g,h)$-tallying bandit problem. Note that by the proof of Lemma~\ref{lem:bandit_mdp}, and by the assumption that $m \leq \sqrt{T}$, after executing $\pi$ for $n \geq 1$ periods we have arrived at the state $(\pi_{n \sqrt{T} - m + 1} \dots \pi_{n \sqrt{T}})$. And since $\pi$ is $\sqrt{T}$-cyclic we are guaranteed that
$$
(\pi_{n \sqrt{T} - m + 1} \dots \pi_{n \sqrt{T}}) = (\pi_{\sqrt{T} - m + 1} \dots \pi_{\sqrt{T}}).
$$
So if we execute $\pi$ for one more period (i.e., the $(n+1)$th period) from this starting state, then regardless of $n$ we will observe an identical sequence of states, since the transition function of the MDP $M$ is deterministic. Hence, the quantity
$$
\sum_{t=n \sqrt{T} + 1}^{(n+1)\sqrt{T}} \mathbb{I}(\pi_t = x) \cdot \mathbb{I} \left( y = \sum_{t'=\max \{ 1, t-m+1 \}}^t \mathbb{I}(\pi_{t'} = x) \right),
$$
used to define $N_{xy}(\pi)$ is independent of $n$, ensuring that $N_{xy}(\pi)$ is well defined. It remains to establish the claim that $N_{xy}(\pi)$ is independent of the action sequence that was played before $\pi$ was executed for $n+1$ periods. Denote this prior action sequence as $a_{1:k}$ for any finite value of $k$. Note that regardless of what $a_{1:k}$ is, after we execute $\pi$ for $n$ periods we still arrive at the state
$$
(\pi_{n \sqrt{T} - m + 1} \dots \pi_{n \sqrt{T}}) = (\pi_{\sqrt{T} - m + 1} \dots \pi_{\sqrt{T}}),
$$
in the MDP $M$. Then if we execute $\pi$ for one more period (i.e., the $(n+1)$th period) from this starting state, then regardless of $a_{1:k}$ we will observe an identical sequence of states, since the transition function of the MDP $M$ is deterministic. This ensures that $N_{xy}(\pi)$ is well defined. \hfill \qeddhruv

\subsection{Proof of Lemma~\ref{lem:bandit_mdp}}
\label{app_proof:lem:bandit_mdp}
\noindent Given an $(m,g,h)$-tallying bandit problem, for ease in notation let the (finite) action set $\ActSet$ be denoted as $\{ 1, 2 \dots \numact \}$. To define the MDP $M$, let its state space be $(\{ 0 \} \cup \ActSet)^m$ and let its action space be $\ActSet$. Let the initial state be the length $m$ vector $(0, 0 \dots 0)$. For each state $s = (s_1, s_2 \dots s_m)$ and action $i$, define the deterministic transition function $\mathcal{T}_M: (\{ 0 \} \cup \ActSet)^m \times \ActSet \to (\{ 0 \} \cup \ActSet)^m$ of the MDP as
$$
\mathcal{T}_M( s, i ) = (s_2, s_2 \dots s_{m-1}, s_m, i).
$$
Let the (finite) horizon of the MDP be the time horizon $T$ of the tallying bandit problem. Finally, define for each state $s$ the (expected) reward function $R_M: (\{ 0 \} \cup \ActSet)^m \to [0,1]$ of the MDP as
$$
R_M(s) = 1 - h_{s_m} \left( \sum_{t=1}^m \mathbb{I}(s_{t} = s_m) \right).
$$
It is immediate the taking actions in the tallying bandit problem corresponds to taking actions in the state space of this MDP. \hfill \qeddhruv

\section{Proof of Theorem~\ref{thm:lower}}
\label{app:lower_bound_proof}
First note that as an immediate consequence of Proposition~\ref{prop1}, we must have $\E \left[ \cpReg \right] \geq  m \numact / 128$. So for the remainder of the proof, we focus on showing that $\E \left[ \cpReg \right] \geq c \sqrt{m \numact T}$ for some numerical constant $c > 0$. Also note that due to the result of Proposition~\ref{prop1}, we can assume for this proof that $m \leq T / 100$, because the complete policy regret scales linearly with $T$ in the regime that $m > T/100$. \\

\noindent At a high level, our proof will proceed via a reduction to best arm identification in stochastic multi armed bandit problems~\citep{slivkins19}. Roughly speaking, we will show the existence of an $(m,g,h)$-tallying bandit problem, such that minimizing the complete policy regret in this problem is at least as hard as identifying the best arm in a stochastic multi armed bandit problem with $\bigtheta \left( m \numact \right)$ arms. \\

\noindent To this end, we first recall Lemma~\ref{lem:bandit_mdp}, which was originally stated at the beginning of Appendix~\ref{app:upper} and proved in Appendix~\ref{app_proof:lem:bandit_mdp}. We have restated it here for convenience since it shall be useful for our proof of Proposition~\ref{prop1}.

\mdplem*

\noindent Now, construct an $(m,g,h)$-tallying bandit problem using the following procedure, where we assume that $m$ is at least some sufficiently large universal constant. Sample $(x^\star, y^\star)$ uniformly at random from $\ActSet \times \{ 23 m/24, 23 m/24 + 1 \dots m \}$. Define $h_x(y) = 1/2$ for each $(x,y) \in \ActSet \times \{ 1, 2 \dots m \}$ such that $(x,y) \neq (x^\star, y^\star)$. Also define $h_{x^\star}(y^\star) = 1/2 - \epsilon$ for some $\epsilon \in (0,1)$ to be specified later. \\

\noindent We now define the stochastic bandit feedback model for this tallying bandit problem as follows. When the player plays action $x$, and this action $x$ has been played a total of $y$ times in the past $m$ timesteps (including the current timestep), then the player receives as feedback a Bernoulli random variable with mean $h_x(y)$. It is immediate that this feedback model meets the criteria outlined in Definition~\ref{def:tallying_bandit}. \\

\noindent Let us now upper bound the cumulative loss incurred by the optimal policy in this tallying bandit problem. To do so, consider the policy $\pi$, which is a length $T$ sequence of actions, that we define as follows. Choose some $x \in \ActSet$ such that $x \neq x^\star$. We define $\pi$ to choose action $x^\star$ for $y^\star$ timesteps and then choose action $x$ for $m - y^\star$ timesteps, and then repeat this length $m$ sequence over and over. Recalling the definition of a $\sqrt{T}$-cyclic policy that was stated in Section~\ref{sec:results_upper_bound}, we can analogously say that $\pi$ is an $m$-cyclic policy, such that within each period of length $m$ it plays $x^\star$ for $y^\star$ times and then plays $x$ for $m - y^\star$ times. \\

\noindent By Lemma~\ref{lem:bandit_mdp}, there exists an MDP $M$ that is equivalent to the constructed tallying bandit problem. Recalling the characterization of this MDP $M$ provided in the proof of Lemma~\ref{lem:bandit_mdp}, let us understand the sequence of states in $M$ that we arrive at when we follow $\pi$. In the first $m$ timesteps, $\pi$ plays $x^\star$ for $y^\star$ times and then plays $x$ for $m - y^\star$ times, and so we arrive at the state
\begin{equation}
\label{eqn:thm_lower_help1}
(x^\star, x^\star \dots x^\star, x, x \dots x) \equiv (x^\star)^{y^\star} \times x^{m - y^\star}.
\end{equation}
Then for the next $y^\star$ timesteps, $\pi$ plays $x^\star$ repeatedly. This leads to the progression of states given by
\begin{equation}
\label{eqn:thm_lower_help2}
\begin{gathered}
(x^\star)^{y^\star - 1} \times x^{m - y^\star} \times (x^\star) \\
(x^\star)^{y^\star - 2} \times x^{m - y^\star} \times (x^\star)^2 \\
(x^\star)^{y^\star - 3} \times x^{m - y^\star} \times (x^\star)^3 \\
\vdots \\
(x^\star) \times x^{m - y^\star} \times (x^\star)^{y^\star - 1} \\
x^{m - y^\star} \times (x^\star)^{y^\star}.
\end{gathered}
\end{equation}
Then for the next $m - y^\star$ timesteps, it plays $x$, so that the state after these timesteps is
$$
(x^\star, x^\star \dots x^\star, x, x \dots x) \equiv (x^\star)^{y^\star} \times x^{m - y^\star},
$$
and we have arrived back at the state listed in Eq.~\eqref{eqn:thm_lower_help1}. \\

\noindent Let us now use this insight to bound the expected loss incurred by following $\pi$. Critically, for every period of $m$ timesteps after the very first period, Eq.~\eqref{eqn:thm_lower_help2} shows that we observe (stochastic instantiations of) the loss value $h_{x^\star}(y^\star)$ for exactly $y^\star \in \{ 23m/24, 23m/24 + 1 \dots m \}$ timesteps. We hence upper bound the expected cumulative loss incurred by $\pi$ as
\begin{equation}
\label{eqn:thm_lower_help3}
\begin{aligned}
(1/2) m + ( (1/2 - \epsilon) y^\star + (1/2)(m - y^\star) )\frac{T - m}{m} &= m/2 + ( -\epsilon y^\star + (1/2) m )\frac{T - m}{m} \\
&= -\epsilon y^\star \frac{T - m}{m} + \frac{T-m}{2} + m/2 \\
&\leq -\epsilon \frac{23m}{24} \frac{T - m}{m} + T/2 \\
&= -\epsilon \frac{23(T - m)}{24} + T/2 \\
&= T/2 - \epsilon 23T/24 + 23 m \epsilon/24.
\end{aligned}
\end{equation}
\noindent Let us now consider the performance of any algorithm which attempts to solve this tallying bandit problem. Instead of the algorithm operating in the usual oracle model, where it might need $m$ actions to arrive at a state that it deems beneficial, let us strengthen the algorithm by equipping it with a generative model. Concretely, we strengthen the algorithm so that at any timestep, it can query any state in $M$ and receive a stochastic instantiation of the $h_x(y)$ loss value corresponding to that state. Given this generative model, it is immediate that an algorithm that is attempting to minimize its cumulative loss (or equivalently, minimize its complete policy regret), is attempting to maximize the number of times it queries states whose loss value is $h_{x^\star}(y^\star)$. \\

\noindent Hence, we can interpret this algorithm as running best arm identification on a classical stochastic multi armed bandit problem with $m \numact /24$ arms. Note that the number of arms in this stochastic multi armed bandit problem is $m \numact / 24$, since $(x^\star, y^\star)$ was sampled uniformly at random from a set of cardinality $m \numact / 24$. The remainder of the argument closely follows that of Slivkins~\citep{slivkins19}. So pick $\epsilon = \sqrt{c m \numact / T}$, for some numerical constant $c > 0$ whose precise value can be found in the proofs of Corollary 2.9 and Theorem 2.10 of Slivkins~\citep{slivkins19}. These two results show that for each timestep less than or equal to $T$, with probability at least $1/12$ the algorithm does not select a state whose loss value is $h_{x^\star}(y^\star)$. In particular, this means that the expected cumulative loss of this algorithm is at least
\begin{equation}
\label{eqn:thm_lower_help4}
(1/12) T/2 + (11/12) T (1/2 - \epsilon) = T/24 + 11 T/24 - 11 \epsilon T /12 = T/2 - 11\epsilon T/12.
\end{equation}
Putting together Eq.~\eqref{eqn:thm_lower_help3} and Eq.~\eqref{eqn:thm_lower_help4}, we hence have that the expected complete policy regret of this algorithm is lower bounded as
\begin{align*}
\E \left[ \cpReg \right] &\geq T/2 - 11\epsilon T/12 - \left( T/2 - \epsilon 23T/24 + 23 m \epsilon/24 \right) \\
&= - 11 \epsilon T/12  + \epsilon 23T/24 - 23 m \epsilon/24 \\
&= \epsilon T/24 - 23 m \epsilon/24 \\
&\geq c \sqrt{m \numact T},
\end{align*}
where in the final step we used our assumption that $m \leq T / 100$, substituted in the definition of $\epsilon = \sqrt{c m \numact / T}$ and redefined the value of the numerical constant $c$. This completes the proof. \hfill \qeddhruv

\section{Proof of Proposition~\ref{prop1}}
\label{app:proposition1}
First, we recall Lemma~\ref{lem:bandit_mdp}, which was originally stated at the beginning of Appendix~\ref{app:upper} and proved in Appendix~\ref{app_proof:lem:bandit_mdp}. We have restated it here for convenience since it shall be useful for our proof of Proposition~\ref{prop1}.

\mdplem*

\noindent We now formally define the algorithm $\mathbb{ALG}_{\text{det}}$ below in Algorithm~\ref{alg:alg_det}. Subsequently, we prove the upper bound in Proposition~\ref{prop1}, and then we prove the lower bound in Proposition~\ref{prop1}.

\begin{comment}
\SetKwComment{Comment}{/* }{ */}
\RestyleAlgo{ruled}
%\SetAlgoNoLine
\LinesNumbered
\begin{algorithm2e}[hbt!]
\label{alg:alg_det}
\caption{$\mathbb{ALG}_{\text{det}}$}
\SetKwInOut{Input}{Inputs}
    \SetKwInOut{Output}{output}
\Input{memory capacity $m$, time horizon $T$}
%\KwResult{$y = x^n$
	\For{$x \in \ActSet$}{
		\For{$y \in \{ 1, 2 \dots m \}$}{
			Choose action $x$. \\
			Observe and store $h_x(y)$.
	}
	}
	Plan (offline) an optimal policy $\overline{\pi} = \{ \overline{\pi}_k \}_{k = m \numact + 1}^{T}$ to play for remaining $T - m \numact$ timesteps. \\
	Choose actions according to $\overline{\pi}$ for remaining $T - m \numact$ timesteps.
\end{algorithm2e}
\end{comment}

\begin{algorithm}[hbt!]
\caption{$\mathbb{ALG}_{\text{det}}$}
\label{alg:alg_det}
\begin{algorithmic}[1]
%\STATE \text{Inputs: horizon length $H$, distribution $\mathcal{D}$, sample size $n$, oracle $\oracle$ as defined in WIO}
%\State Initialize $\mathcal{S}_\pi$ and $\mathcal{A}_\pi$ each as the empty set
%\State Initialize $\pi$ as the empty function from $\mathcal{S}_\pi$ to $\mathcal{A}_\pi$
\Require memory capacity $m$, time horizon $T$
\For{$x \in \ActSet$}
		\For{$y \in \{ 1, 2 \dots m \}$}
			\State Choose action $x$.
			\State Observe and store $h_x(y)$.
\EndFor
\EndFor
	\State Plan (offline) an optimal policy $\overline{\pi} = \{ \overline{\pi}_k \}_{k = m \numact + 1}^{T}$ to play for remaining $T - m \numact$ timesteps.
	\State Choose actions according to $\overline{\pi}$ for remaining $T - m \numact$ timesteps.
\end{algorithmic}
\end{algorithm}

\subsection{Proof of Upper Bound in Proposition~\ref{prop1}}
\noindent First note that the offline planning step in Algorithm~\ref{alg:alg_det} is statistically (although perhaps not computationally) feasible, since the player has full information about the loss functions once it stores $h_x(y)$ for each $(x,y) \in \ActSet \times \{ 1, 2 \dots m \}$. \\

\noindent Let $\pi^{\star \star}$ denote the optimal policy for the problem, so that $\pi^{\star \star} \in \ActSet^T$ is a length $T$ sequence of actions such that following this sequence achieves the minimum loss. Let $\ell_t(\pi^{\star \star})$ denote the loss incurred at the $t$th timestep while playing the action sequence $\{ \pi^{\star \star}_t \}_{t=1}^{T}$. Let $\ell(\overline{\pi})$ denote the total loss incurred by the final step of Algorithm~\ref{alg:alg_det} which starts playing $\overline{\pi}$ at timestep $m \numact + 1$.  \\

\noindent To complete the proof, we will leverage the MDP characterization of tallying bandit problems that was established in the proof of Lemma~\ref{lem:bandit_mdp}. Let $M$ denote the equivalent MDP for this tallying bandit problem. Observe that $\overline{\pi}$ is the optimal policy for the remaining $T - m \numact$ timesteps assuming that $\overline{\pi}$ is initialized at the initial state $s = (\numact, \numact \dots \numact)$ in this MDP $M$. Also note that playing the action sequence $\{ \pi^{\star \star}_t \}_{t=m \numact + 1}^T$ from the state $s$ leads to a cumulative loss that is upper bounded by
$$
m + \sum_{t= m (\numact+1) + 1}^T \ell_t(\pi^{\star \star}),
$$
which follows because regardless of the initial state, we arrive at state $s' = (\pi^{\star \star}_{m \numact + 1}, \pi^{\star \star}_{m \numact + 2} \dots \pi^{\star \star}_{m (\numact+1)} )$ in $M$ after playing the length $m$ action sequence $\{ \pi^{\star \star}_t \}_{t=m \numact + 1}^{m (\numact + 1)}$. From state $s'$, the losses we experience when playing $\{ \pi^{\star \star}_t \}_{t=m (\numact+1) + 1}^T$ at each timestep $t \geq m (\numact+1) + 1$ are exactly $\ell_t(\pi^{\star \star})$. The optimality of $\overline{\pi}$ hence implies that
\begin{equation}
\label{eqn:help_prop_upper}
\ell(\overline{\pi}) \leq m + \sum_{t= m (\numact+1) + 1}^T \ell_t(\pi^{\star \star}).
\end{equation}
Hence, by naively upper bounding the loss incurred in the first $m \numact$ timesteps, the complete policy regret of Algorithm~\ref{alg:alg_det} is upper bounded via Eq.~\eqref{eqn:help_prop_upper} as
\begin{align*}
\cpReg &\leq m \numact + \ell(\overline{\pi}) - \sum_{t= m \numact + 1}^T \ell_t(\pi^{\star \star}) \\
&\leq m \numact + m + \sum_{t= m (\numact+1) + 1}^T \ell_t(\pi^{\star \star}) - \sum_{t= m \numact + 1}^T \ell_t(\pi^{\star \star}) \\
&= m \numact + m - \sum_{t= m \numact + 1}^{m (\numact + 1)} \ell_t(\pi^{\star \star}) \\
&\leq (m+1) \numact.
\end{align*}
This verifies the upper bound in Proposition~\ref{prop1}. \hfill \qeddhruv

\subsection{Proof of Lower Bound in Proposition~\ref{prop1}}
Construct an $(m,g,h)$-tallying bandit problem via the following procedure. Sample an action $x^\star$ uniformly at random from $\ActSet$, and keep its identity hidden from the user. Define the functions $\{ h_x \}_{x \in \ActSet}$ as
$$
h_{x^\star}(y) = \begin{cases} 1 \text{ if } y < m \\ 0 \text{ if } y = m \end{cases} \text{ and } h_x = 1 \text{ if } x \neq x^\star.
$$
It is immediate that the optimal policy always plays the action $x^\star$, and its cumulative loss is precisely $m-1$. Meanwhile, to obtain zero loss at any timestep, the player must identify $x^\star$. Note that to identify whether a certain action $x$ equals $x^\star$, the player must play $x$ for $m$ consecutive times, in order to receive the feedback $h_x(m)$. Since there are $\numact$ actions, and identifying whether an action is the correct one requires $m$ queries, a standard counting argument~\citep{du20lowerbound} reveals that (in expectation over a possibly randomized strategy) the player makes at least $m \numact / 64$ queries before observing $h_{x^\star}(m)$. Hence, the (expected) complete policy regret of the (possibly randomized) player is lower bounded as
$$
\E \left[ \cpReg \right] \geq m \numact / 64 - (m-1) \geq m \numact / 128,
$$
where we assume that $\numact$ is larger than some numerical constant. This verifies the lower bound in Proposition~\ref{prop1}. \hfill \qeddhruv

\begin{comment}
\newpage
\begin{dhruvlemma}
Let $x_1, x_2 \dots x_n \in [0,1]$ be such that $\sum_{i=1}^n x_i = 1$. Then we have that
$$
\sum_{i=1}^n \sqrt{x_i} \leq \sqrt{n}.
$$
Moreover, this bound is tight.
\end{dhruvlemma}

\noindent Let $y = [\sqrt{x_1}, \sqrt{x_2} \dots \sqrt{x_n}]$. Note that
$$
1 = \sum_{i=1}^n x_i = \sum_{i=1}^n \sqrt{x_i} \sqrt{x_i} = y^T y = \| y \|_2^2 \implies \| y \|_2 = 1.
$$
Also note that
$$
\sum_{i=1}^n \sqrt{x_i}  = \| y \|_1,
$$ 
and so it is sufficient to bound $\| y \|_1$. Then observe that
$$
\| y \|_1 \leq \sqrt{n} \| y \|_2 = \sqrt{n} \implies \sum_{i=1}^n \sqrt{x_i} \leq \sqrt{n}.
$$
To see that this bound is tight, observe that for the choice $x_1 = x_2 = \dots x_n = 1/n$, we have that
$$
\sum_{i=1}^n \sqrt{x_i} = \sum_{i=1}^n 1/\sqrt{n} = \sqrt{n}.
$$

\begin{dhruvlemma}
We need $S$ such that
$$
T = \sum_{s=1}^S 2 n_s \numact m \sqrt{T}.
$$
\end{dhruvlemma}
\begin{proof}
Note that
$$
T = \sum_{s=1}^S 2 n_s \numact m \sqrt{T} = 2 \numact m \sqrt{T} \sum_{s=1}^S 2^s = 4 \numact m \sqrt{T} (2^S - 1) .
$$
So we have that
$$
2^S - 1 = \frac{\sqrt{T}}{4 \numact m} \implies S = \log_2 \left( \frac{\sqrt{T}}{4 \numact m} + 1 \right).
$$
\end{proof}

\end{comment}

\bibliographystyle{alpha}
\bibliography{colt22v2refs}

\end{document}